\def\year{2022}\relax

\documentclass[letterpaper]{article} % DO NOT CHANGE THIS
\usepackage{aaai22}  % DO NOT CHANGE THIS
\usepackage{times}  % DO NOT CHANGE THIS
\usepackage{helvet}  % DO NOT CHANGE THIS
\usepackage{courier}  % DO NOT CHANGE THIS
\usepackage{xcolor}
\usepackage[hyphens]{url}  % DO NOT CHANGE THIS
\usepackage{graphicx} % DO NOT CHANGE THIS
\usepackage{natbib}  % DO NOT CHANGE THIS AND DO NOT ADD ANY OPTIONS TO IT
\usepackage{caption} % DO NOT CHANGE THIS AND DO NOT ADD ANY OPTIONS TO IT
\DeclareCaptionStyle{ruled}{labelfont=normalfont,labelsep=colon,strut=off} % DO NOT CHANGE THIS
\usepackage[utf8]{inputenc}
\usepackage{graphicx,subfigure}
\usepackage{algorithm}
\usepackage{algorithmic}
\usepackage{amsfonts}
\usepackage{pifont}% http://ctan.org/pkg/pifont
\newcommand{\cmark}{\ding{51}}%
\newcommand{\xmark}{\ding{55}}%
\usepackage{lipsum}
\usepackage{graphicx}
\usepackage{mathrsfs}
\usepackage{amsmath}
\usepackage{amssymb}
\usepackage{caption}
\usepackage{booktabs}
\usepackage[utf8]{inputenc}
\usepackage[english]{babel}
\usepackage{bbm}
\usepackage[algo2e]{algorithm2e}
\usepackage{amsthm}
\newtheorem{theorem}{Theorem}

\theoremstyle{definition}
\newtheorem{definition}{Definition}

\usepackage{newfloat}
\usepackage{listings}
\floatstyle{ruled}
\newfloat{listing}{tb}{lst}{}
\floatname{listing}{Listing}
\title{SimSR: Simple Distance-based State Representation\\ for Deep Reinforcement Learning}
\author {
    Hongyu Zang \textsuperscript{\rm 1},
    Xin Li \textsuperscript{\rm 1},
    Mingzhong Wang \textsuperscript{\rm 2}
}
\affiliations {
    \textsuperscript{\rm 1} Beijing Institute of Technology\\
    \textsuperscript{\rm 2} University of the Sunshine Coast\\
    \{zanghyu,xinli\}@bit.edu.cn, mwang@usc.edu.au
}

\usepackage{bibentry}
\usepackage{amsthm}
\theoremstyle{proposition}
\newtheorem{proposition}{Proposition}
\begin{document}

\maketitle

\begin{abstract}
This work explores how to learn robust and generalizable state representation from image-based observations with deep reinforcement learning methods. Addressing the computational complexity, stringent assumptions and representation collapse challenges in existing work of bisimulation metric, we devise \textit{Simple State Representation} (SimSR) operator. SimSR enables us to design a stochastic approximation method that can practically learn the mapping functions (encoders) from observations to latent representation space. In addition to the theoretical analysis and comparison with the existing work, we experimented and compared our work with recent state-of-the-art solutions in visual MuJoCo tasks. The results shows that our model generally achieves better performance and has better robustness and good generalization.
\end{abstract}

\section{Introduction}
Deep reinforcement learning (RL) with image-based observations commonly leverages deep convolutional networks to obtain low-dimensional representations of environments to accomplish sequential decision tasks. Conventionally, the features, which deep convolutional networks embed raw observations into, are modeled as state representations. As these representations are generally learned implicitly as a byproduct of deep RL, they are unlikely to provide sufficient and compact information to summarize the task-relevant states or enhance policy learning. This explains in part why learning policies from physical states, e.g., control inputs like robot's velocity, is generally more sample-efficient than learning policies from image-based environment. 

However, since physical states are unobtainable in many real-world domains, recent work in RL has attempted to learn a mapping from image space to low-dimensional representation space that contains the concise state information with the help of auxiliary tasks and self-supervision objectives~\cite{ DBLP:conf/icml/PathakAED17, DBLP:conf/icml/LaskinSA20, DBLP:conf/nips/LeeNAL20, DBLP:conf/aaai/Yarats0KAPF21, DBLP:journals/corr/abs-2106-04379, DBLP:journals/corr/abs-2102-13268}. Some existing work learns low-dimensional representation by enforcing reconstruction objective to improve the predictiveness and consistency of representations~\cite{DBLP:conf/iclr/JaderbergMCSLSK17, DBLP:conf/icml/PathakAED17, DBLP:conf/aaai/Yarats0KAPF21}. Other
work utilizes contrastive loss for heuristic augmentation pairs to achieve more distinguishable representations of the environment observations~\cite{DBLP:journals/corr/abs-1807-03748, DBLP:conf/nips/AnandROBCH19, DBLP:conf/iclr/Warde-FarleyWKI19, DBLP:conf/icml/LaskinSA20}. However, both lines of research are prone to task-agnostic representations as they encode all signals no matter which are relevant to the task or not.

The latest study~\cite{DBLP:conf/icml/GeladaKBNB19, DBLP:conf/aaai/Castro20, DBLP:conf/iclr/0001MCGL21, castro2021mico} presents promising results on learning robust state representation with bisimulation metric~\cite{DBLP:conf/uai/FernsP14, DBLP:conf/aaai/Castro20}, which implicates the state equivalence/similarity and thus is commonly used for MDP (Markov Decision Process) model reduction/minimization. Though driven by various motivations, the use of bisimulation metric in RL can be thought of as learning suitable measurement in the representation space. These work trains the mapping functions (encoders) to minimize the discrepancy between the representations of similar observation pairs in the latent representation space according to their bisimulation metric. Intuitively, the distance between two encoded observations should correspond to their ``behavioral difference'' that is jointly determined by the difference of their immediate rewards and the difference of the distributions over their next states. Although effective, the bisimulation metric needs to compute the Wasserstein distance between distributions, which is computationally expensive.

Existing work only partially addressed the complexity issue, resulting in limitations in deploying the solutions. 
\citet{DBLP:conf/icml/GeladaKBNB19} require a stringent assumption that the learned representation is Lipschitz, so that the dual form of Wasserstein distance can be used to alleviate the computational complexity. \citet{DBLP:conf/iclr/0001MCGL21} assume the state distribution is Gaussian, so that it can use euclidean distance of the Gaussian distribution in the latent space to compute the closed-form Wasserstein distance, thus optimizing the $\ell_1$ distance between two representations. However, the inconsistency of $\ell_1$ distance and euclidean distance may steer inaccurate approximation. \citet{castro2021mico} develop MICo operator to update a sample-based metric that does not contain Wasserstein distance. However, this design has the risk to fall in the failure mode of representation collapse since it violates the ``zero self-distance'' property.

Therefore, we argue that it is crucial to find a measurement in the representation space that is: (1) functionally equivalent to the bisimulation metric, (2) simpler in computational complexity, and (3) insusceptible to representation collapse issue. Intuitively, such a measurement, which plays the same role as bisimulation metric in training encoder, should improve the learning efficiency
in continuous and high-dimensional observation space while maintaining the ability of measuring the ``long-term behavior'' of states and reserving good theoretical properties.

In this paper, we propose \textit{Simple State Representation} (SimSR) to achieve task-relevant latent state representation while satisfying the aforementioned three features.
Specifically, we adapt bisimulation metric with a cosine-based operator, which is referred to as SimSR operator in the sequel, to measure the behavioral difference of two states in the representation space. SimSR aims to guarantee that the observations are embedded in a better representation space. Since SimSR involves the cosine distance, the projected latent space can be viewed as a unit sphere, where the biases induced by the scale of the ``behavioral difference'' between states, especially the scale of the reward, can be removed.
Besides, since cosine distance has a guarantee of ``zero self-distance'', SimSR theoretically avoids the case that states with different values collapse to the same representation in our proposed framework. 

To avoid the computation of expectation, SimSR operator applies sample-based computation, which requires the sampling of the next state representations. This can be achieved by either 1) sampling the next observations and encoding them into the latent representation space or 2) encoding the current observations into the latent space and learning the latent transition model explicitly to compute the next state representations. The former approach is more convenient in implementation but it may involve redundant information. In comparison, the latter is able to provide more compact transition dynamics of the entire environment, resulting in better sample efficiency. To provide a comprehensive study, we developed and experimented two types of SimSR update operator for each approach. For the latter operator, we utilized an ensemble version of Gaussian distribution to facilitate the latent transition modelling.

In theory, SimSR confines the mapping to a set of functions yielding a low-dimensional state representation space, where the distance between the state representations is consistent to the ``behavioral difference'' of the states. In practice, neural networks can be utilized to minimize the SimSR loss to find one of such mapping functions. More specifically, we utilized the momentum encoder~\cite{DBLP:conf/cvpr/He0WXG20} to stabilize the learning process.

Our contributions can be summarized as follows:

\begin{enumerate}
    \item We adapt bisimulation metric coupled with cosine-based measurement to align the state representations' similarity with their behavioral consistency to obtain ``good'' state representations for RL tasks.
    \item To remove the redundant information from the observation space, we additionally construct a compact latent dynamics model to facilitate the sampling procedure. The proposed SimSR loss is optimized to learn state representations from high-dimensional pixel-based input. 
    \item We empirically experimented SimSR and demonstrated that it can (1) achieve comparable and better/robust performance over existing methods on a set of continuous control benchmarks, (2) learn more robust representation against task-irrelevant information, and (3) potentially generalize to unseen tasks.
\end{enumerate}

\section{Related Work}

\subsection{Representation learning in RL}
A good representation of state is crucial to the success of applying RL. 
There exist previous work ~\cite{DBLP:conf/ijcnn/LangeR10, DBLP:conf/ijcnn/LangeRV12, DBLP:conf/nips/LeeNAL20, DBLP:conf/aaai/Yarats0KAPF21} devoted to training deep auto-encoders with reconstruction loss to learn low-dimensional representations and improve the policy performance in visual reinforcement learning.
Some other RL methods~\cite{DBLP:conf/iclr/ShelhamerMAD17, DBLP:conf/icml/HafnerLFVHLD19, DBLP:conf/nips/LeeFLGLCG20} proposed to learn state representation from predictive loss, yet suffering from the additional complexity of balancing various auxiliary losses. More recently, some work~\cite{DBLP:conf/icml/LaskinSA20,DBLP:conf/iclr/YaratsKF21, DBLP:conf/icml/YaratsFLP21, DBLP:conf/icml/StookeLAL21} has used contrastive losses to learn representations and achieved considerable performance improvement. However, all these methods capture all features in the observation space no matter they are relevant to the downstream tasks or not, which in turn may degrade the policy performance.

\subsection{Bisimulation}
Early work on overcoming the curse of dimensionality~\cite{DBLP:journals/ai/GivanDG03, DBLP:conf/isaim/LiWL06} defined equivalence relations of the states to aggregate states and reduce the system complexity. However, the equivalence defined on the stochastic processes is impractical since it requires the transition distribution to be exactly the same. \citet{ DBLP:conf/uai/FernsPP04,DBLP:conf/uai/FernsP14} addressed the issue with bisimulation metrics, which measure how similar two states are and can also be used as a distance function to aggregate states more easily. 
Recently, \citet{DBLP:conf/icml/GeladaKBNB19} have explored the use of a variety of metrics and showed their $l_2$ distance representation has the upper bounds as the bisimulation distance, while requiring the stringent assumption of the representations being Lipschitz. In consideration that bisimulation metric may result in improper comparison between states, \citet{DBLP:conf/aaai/Castro20} developed on-policy bisimulation metric, which unfortunately still suffers the difficulty of computing the intractable Wasserstein distance.
As the successors of on-policy bisimulation metric,
\citet{DBLP:conf/iclr/0001MCGL21} modeled the latent dynamic transition as Gaussian distribution to alleviate the computational complexity of Wasserstein distance, while
\citet{castro2021mico} introduced the MICo distance to approximate the bisimulation metric. Both of them are restrictive as either having the stringent assumption and the distance inconsistency, or being prone to the failure mode of representation collapse. In comparison, we devised a simple and flexible measurement to characterize the behavioral similarity between states without using the Wasserstein distance, and we proved that our approach can address the aforementioned issues both theoretically and experimentally. 

\section{Preliminaries}

This section explains the basic notations in modelling RL problems and bisimulation.  

\subsection{Problem setting}
This paper focuses on learning from the environment with infinite image-based observations. We consider the underlying RL problems as a \textbf{block Markov decision process} (BMDP) \cite{DBLP:conf/icml/DuKJAD019}, which is represented as a tuple $\left\langle\mathcal{X}, \mathcal{S}, \mathcal{A}, r, \phi, q, \mathcal{P}, \mathcal{P}_0, \gamma\right\rangle$ with a finite unobservable state space $\mathcal{S}$, action space $\mathcal{A}$, and possibly infinite but observable space $\mathcal{X}$. With the assumption that each observation $\mathbf{x}\in \mathcal{X}$ uniquely determines its generating state $\mathbf{s} \in \mathcal{S}$, we can obtain the latent state regarding its observation by a projection function $\phi(\mathbf{x}):\mathcal{X}\rightarrow\mathcal{S}$. Therefore, $\mathbf{s}$ and $\phi(\mathbf{x})$ can be used interchangeably. The dynamics of a BMDP is described by the initial latent state distribution $\mathcal{P}_0$ and the state-transition probability function $\mathcal{P}$ which decides the next latent state of the agent $\mathbf{s}'\sim \mathcal{P}(\mathbf{s}'|\mathbf{s},\mathbf{a})$. The corresponding transition function under the observation space is defined as $\mathbf{x}'\sim \hat{\mathcal{P}}(\mathbf{x}'|\mathbf{x},\mathbf{a})$, where $\hat{\mathcal{P}}(\mathbf{x}'|\mathbf{x},\mathbf{a})= q(\mathbf{x}'|\mathbf{s}') \mathcal{P}(\mathbf{s}'|\mathbf{s},\mathbf{a})$ and $q$ 
indicates the probability of an latent state $\mathbf{s}$ presenting itself as an observation $\mathbf{x}$. 
The agent in a latent state $\mathbf{s} \in\mathcal{S}$ selects an action $\mathbf{a}\in\mathcal{A}$ according to the policy $\pi(\mathbf{a}|\mathbf{s})$. 

Following the assumption that the projections from the observation space to the latent state space are deterministic and the environment always provides the reward function depending on the observations, the performance of the state-action pair is quantified by the reward function $r(\mathbf{x},\mathbf{a})$ given by the environment. 
Similarly, the latent space reward $r(\mathbf{s},\mathbf{a})$ can be defined based on the environment reward $r(\mathbf{x},\mathbf{a})$ with the transition models.
$\gamma$ is a discount factor ($0<\gamma<1$) which quantifies how much value we weigh for future rewards. 
The goal of the agent is to find the optimal policy $\pi(\mathbf{a}|\mathbf{s})$ to maximize the expected reward $\mathbb{E}_{\mathbf{x}_{0}, \mathbf{a}_{0}, \ldots}\left[\sum_{t=0}^{\infty} \gamma^{t} r\left(\mathbf{x}_{t}, \mathbf{a}_{t}\right)\right]$. 
Besides, we approximate stacked pixel images as the observations in the context.

\subsection{Bisimulation metric} \label{bisimulation}
Bisimulation measures equivalence relations on MDPs with a recursive form: two states are deemed equivalent if they share the equivalent distributions over the next equivalent states and they have the same immediate reward~\cite{DBLP:conf/popl/LarsenS89, DBLP:journals/ai/GivanDG03}. However, since bisimulation considers equivalence for all actions, including bad ones, it commonly results in “pessimistic” outcomes. Instead, \citet{DBLP:conf/aaai/Castro20} developed \textbf{$\pi$-bisimulation} which removes the requirement of considering each action and only needs to consider the actions induced by a policy $\pi$.

\begin{definition}\cite{DBLP:conf/aaai/Castro20}
\label{def_pi_bisim}
Given an MDP $\mathcal{M}$, an equivalence relation $E^{\pi} \subseteq \mathcal{S} \times \mathcal{S}$ is a $\pi$-bisimulation relation if whenever $(\mathbf{s},\mathbf{u})\in E^{\pi}$ the following properties hold:
\begin{enumerate}
    \item $r(\mathbf{s},{\pi})=r(\mathbf{u},{\pi})$
    \item $\forall C \in \mathcal{S}_{E^{\pi}}, P(C|\mathbf{s},{\pi})=P(C|\mathbf{u},{\pi})$
\end{enumerate}
where $\mathcal{S}_{E^{\pi}}$ is the state space $\mathcal{S}$ partitioned into equivalence classes defined by $E^{\pi}$. Two states $\mathbf{s},\mathbf{u} \in S$ are \textbf{$\pi$-bisimilar} if there exists a \textbf{$\pi$-bisimulation relation} $E^{\pi}$ such that $(\mathbf{s}, \mathbf{u}) \in E^{\pi}$. 
\end{definition}
However, $\pi$-bisimulation is still too stringent to be applied at scale as $\pi$-bisimulation relation emphasizes the equivalence is a binary property: either two states are equivalent or not, thus becoming too sensitive to perturbations in the numerical values of the model parameters. The problem becomes even more prominent when deep frameworks are applied. 

Thereafter, \citet{DBLP:conf/aaai/Castro20} proposed a \textbf{$\pi$-bisimulation metric} to leverage the absolute value between the immediate rewards w.r.t. two states and the $1$-Wasserstein distance ($\mathcal{W}_1$) between the transition distributions conditioned on the two states and the policy $\pi$ to formulate such measurement.

Although the Wasserstein distance is a powerful metric to calculate the distance between two probability distributions, it requires to enumerate all states which is impossible in RL tasks of continuous state space. Thus, \citet{DBLP:conf/aaai/Castro20} only considered the deterministic MDP problems to avoid the computation of Wasserstein distance, and rewrite the operator $\mathcal{F}^{\pi}$ with a form of  $\mathcal{F}^{\pi}(d)(\mathbf{s}, \mathbf{u})=|r_{\mathbf{s}}^\pi-r_{\mathbf{u}}^\pi|+\gamma d\left(s', u'\right)$, where $s'$ and $u'$ are the deterministic next states of an agent starting from state $\mathbf{s}$ and $\mathbf{u}$, respectively, under the policy $\pi$.

Various extensions have been proposed \cite{DBLP:conf/icml/GeladaKBNB19,DBLP:conf/iclr/0001MCGL21, castro2021mico} to reduce the computational complexity. A representative approach is MICo distance \cite{castro2021mico}, which restricts the coupling class to the independent coupling and thus can be oblivious to the intractable Wasserstein distance. The MICo operator and its associated theoretical guarantee is given as: 
\begin{theorem}
\label{theorem:MICo_dis}\cite{castro2021mico}
Let the reward $r_{\mathbf{x}}^{\pi}=\sum_{\mathbf{a}\in\mathcal{A}}\pi(\mathbf{a}|\mathbf{s})r(\mathbf{x},\mathbf{a})$, the transition $\hat{\mathcal{P}}=\sum_{\mathbf{a}\in\mathcal{A}}\pi(\mathbf{a}|\mathbf{s})\hat{\mathcal{P}}(\cdot|\mathbf{x},\mathbf{a})$. Given a policy $\pi$, MICo distance $\mathscr{F}^{\pi}\colon \mathbb{R}^{\mathcal{S}\times \mathcal{S}}\rightarrow\mathbb{R}^{\mathcal{S}\times \mathcal{S}}$ as:
\begin{equation}
\begin{aligned}
    \label{ori_trans}
    \mathscr{F}^{\pi}U(\mathbf{x},\mathbf{y})=&|r_{\mathbf{x}}^{\pi}-r_{\mathbf{y}}^{\pi}|+\\\gamma &\mathbb{E} _{\mathbf{x}'\sim \hat{\mathcal{P}}_{\mathbf{x}}^{\pi},\mathbf{y}'\sim \hat{\mathcal{P}}_{\mathbf{y}}^{\pi}}[U(\mathbf{x}',\mathbf{y}')]
\end{aligned}
\end{equation}
has a fixed point $U^{\pi} \colon \mathcal{S}\times \mathcal{S}\rightarrow \mathbb{R}$.
\end{theorem}
However, we argue that repeatedly applying MICo operator requires a non-trivial necessity of MICo distance being a diffuse metric. Such a requirement diminishes the supposed practical advantage of MICo and strains its flexibility and facilitation in application. 
In this work, we focus on learning more effective state representation, avoiding representation collapse, and therefore developing a more concise update operator, the SimSR operator, to circumvent the aforementioned limitations.

\section{Simple State Representation (SimSR) Framework}
This section defines Simple State Representation (SimSR) operator. SimSR operators utilize cosine distance to measure the behavioral difference between state representations. Consequently, SimSR loss is introduced to construct the mapping function and the state representation from a learning perspective.

\subsection{SimSR Operator}
\label{SimSR_distance}

Our goal is to learn robust latent state representation to benefit RL with boosting efficiency. 
The intuition is the bisimular states should have ``similar'' or even the same representations in a low-dimensional feature space. Therefore, the bisimulation metric is applied to guide the learning of the mapping function. As MICo distance requires itself to be a diffuse metric, we consider if relaxing the restriction so that a non-diffuse metric can still reach the same fixed point. We define the measurement between two latent states with cosine distance (normalized dot product distance) as the base ``metric'', expressible as:
\begin{equation}
    \overline{\text{cos}}_\phi(\mathbf{x},\mathbf{y}) = 1-\cos_{\phi}(\mathbf{x},\mathbf{y}) = 1-\frac{\phi(\mathbf{x})^T\cdot \phi(\mathbf{y})}{\|\phi(\mathbf{x})\|\cdot\|\phi(\mathbf{y})\|}.
\end{equation}
We therefore deliberately construct the update operator $\mathscr{F}^{\pi}$ which is based on cosine distance but still has the same fixed point as MICo.

\begin{theorem}
\label{theorem:SimSR_op}
Given a policy $\pi$, Simple State Representation (SimSR) which is updated as:
\begin{equation}
\begin{aligned}
    \label{ori_trans}
    \mathscr{F}^{\pi}\overline{\text{cos}}_{\phi}(\mathbf{x},\mathbf{y})=&|r_{\mathbf{x}}^{\pi}-r_{\mathbf{y}}^{\pi}|+\\\gamma &\mathbb{E} _{\mathbf{x}'\sim \hat{\mathcal{P}}_{\mathbf{x}}^{\pi},\mathbf{y}'\sim \hat{\mathcal{P}}_{\mathbf{y}}^{\pi}}[\overline{\text{cos}}_{\phi}(\mathbf{x}',\mathbf{y}')]
\end{aligned}
\end{equation}
has the same fixed point as MICo\footnote{All proofs are provided in Appendix.}.
\end{theorem}

The second term of the right-hand side in Eq.~\ref{ori_trans} is computed by sample-based approximation to avoid computing W-distance. SimSR operator can be applied iteratively on an arbitrary initialized $\phi$ to update the corresponding representation, leading to a set of mapping functions, which can serve as the encoders to embed the observations into the state representations while satisfying the constraints regarding the cosine distance. 

Theoretically, SimSR operator is built on the extensive prior research on bisimulation-based metrics. We believe that SimSR distinguishes itself from other bisimulation-based methods since the literature usually learns state representation by approximating a well-designed specific metric. Instead, we do not try to approximate metrics, but simply utilize cosine distance as a ``base metric'' and guide the state representation learning to adapt such metric. In principle, different choices can be viewed as different additional restrictions to the states. For example, MICo distance restricts the measurement between states being diffuse metric, SimSR restricts the state feature to be embedded in a unit sphere. The motivations and benefits of the state representation that learned from SimSR operator are outlined below:

\subsubsection{\textit{Unit length}}  With the cosine distance as the base ``metric'' in SimSR, all learned state features will be scaled to unit length. Without normalization, the distance may be dominated by the scale of state features, resulting in negative feedback to state representations and disturbance to value approximation and policy learning. Cosine distance effectively restricts the output space to the unit sphere and guarantees that all features contribute equally to the resulting distance. \citet{DBLP:conf/icml/BojanowskiJ17,DBLP:conf/mm/WangXCY17,DBLP:conf/icml/0001I20} demonstrated that normalized outputs lead to superior representations in their representation learning work. Our experiments also proved that the unified features perform better than those that not unified.

\subsubsection{\textit{Zero self-distance}} 
``Representation collapse'' in RL means that two states (observations) with different values are collapsed to the same representation, e.g., $\phi(\mathbf{x})\rightarrow 0$ for all $\mathbf{x}$. State representations are prone to representation collapse when they are jointly trained with value functions and policy functions without providing ground states. Let $\phi^{\pi}$ be the mapping function set corresponding to the fixed point. With the zero self-distance property guaranteed by cosine distance, we can assure that collapse issue is negligible for $\phi^{\pi}$ as two observations are encoded with the same representation only if they have same values. More detailed discussion is provided in Appendix.

\subsubsection{\textit{Low computational complexity}}
SimSR operator $\mathscr{F}^{\pi}$ does not compute Wasserstein distance, which drastically simplifies the update procedure and results in the same order of computational complexity with MICo operator.

\subsection{SimSR with latent state dynamics}

The target problems of the paper can be modeled as a block MDP with an infinite observation space and a finite latent state space, where the dynamics can be described by either latent state transition function $\mathcal{P}$ or observation transition function $\hat{\mathcal{P}}$. SimSR operator defined in Theorem~\ref{theorem:SimSR_op} is computed by sample-based approximation with the observations sampled from $\hat{\mathcal{P}}$ which is provided by the raw environment. Although sampling from an infinite observation space is convenient, this strategy may inject considerable redundant information to the representation learning and the policy learning process. Therefore, we explicitly construct a more compact dynamics model on the limited latent state space. Consequently, we redefine SimSR operator with the latent state dynamics. The learned latent state dynamics are expected to provide agents more diverse knowledge of the goals, leading to better convergence and robustness in the representation learning and policy learning.

Specifically, we first construct the latent state dynamics model $\mathcal{P}(\cdot|\phi(\mathbf{x}),\mathbf{a})$, then encode the observation $\mathbf{x}$ to its latent state $\phi(\mathbf{x})$ and sample the corresponding next latent state from the dynamics model $\mathbf{s}'\sim \mathcal{P}(\cdot|\phi(\mathbf{x}),\mathbf{a})$, and finally leverage cosine distance to acquire SimSR distance. Theorem~\ref{simsr_with_dynamics} defines the updated operator and proves its guaranteed convergence.

\begin{theorem}
\label{simsr_with_dynamics}
Given a policy $\pi$, let SimSR operator $\mathbbm{F}^{\pi}:\mathbb{R}^{\mathcal{S}\times \mathcal{S}}\rightarrow\mathbb{R}^{\mathcal{S}\times \mathcal{S}}$ be
\begin{equation}
\begin{aligned}
    \label{con_trans}
    \mathbbm{F}^{\pi}\overline{\text{cos}}_{\phi}(\mathbf{x},\mathbf{y})=&|r_{\mathbf{x}}^{\pi}-r_{\mathbf{y}}^{\pi}|+\gamma \mathbb{E} _{\mathbf{s}'\sim \mathcal{P}_{\phi(\mathbf{x})}^{\pi},\mathbf{u}'\sim \mathcal{P}_{\phi(\mathbf{y})}^{\pi}}[\overline{\text{cos}}(s',u')].
\end{aligned}
\end{equation}
If latent dynamics are specified, $\mathbbm{F}^{\pi}$ has a fixed point.
\end{theorem}

\subsection{Latent state dynamics modeling}

Following~\cite{DBLP:conf/iclr/0001MCGL21}, we develop an ensemble version of probabilistic dynamics $\{\mathcal{P}_k(\cdot|\phi(\mathbf{x}),\mathbf{a})\}^K_{k=1}$. Typically, each probabilistic transition dynamics model can be represented as $\mathcal{P}_{\theta_k}(\cdot|\phi(\mathbf{x}),\mathbf{a})=\mathcal{N}(\mu_{\theta_k},\sigma_{\theta_k})$. At the training step, we update the parameters of dynamics models by Gaussian negative log-likelihood loss function:
\begin{equation}
\label{eq:dynamics_loss}
    \mathcal{L}_{\mathcal{P}}(\theta_i)=\frac{1}{K}\sum_{k=1}^{K}\left[\frac{\log \sigma_{\theta_k}^{2}(\phi(\mathbf{x}))}{2}+\frac{\left(\phi(\mathbf{x}')-\mu_{\theta_k}(\phi(\mathbf{x}))\right)^{2}}{2 \sigma_{\theta_k}^{2}(\phi(\mathbf{x}))}\right],
\end{equation}
where $i\in\{1,2,...,K\}$. Since the probabilistic models are all randomly initialized, though they share the same gradient, they generally acquire different parameters after training. Therefore, the ensemble model can estimate the uncertainty of the environments in some sense and is more suitable for approximating the MDPs that with a high degree of stochasticity. At the inference step, to apply theorem 4, we randomly sample one of the $K$ probabilistic dynamics to perform the sampling of next latent state $\mathbf{s}'$ and $\mathbf{u}'$.

\subsection{Representation Learning with SimSR loss}

We adopt two convolution layer with a fully connected layer as the encoder module $\phi$. After encoding the observations, we then $\ell_2$-normalize the output of the fully connected layer to scale the features to unit length. 

To avoid the computation of the expected distance in Eq.~\ref{ori_trans} and Eq.~\ref{con_trans} regarding the probabilistic models, 
we, inspired by temporal-difference learning~\cite{DBLP:books/lib/SuttonB98}, estimate the target of the distance (the R.H.S of these two equations) by stochastic estimation methods : a sample distance between next latent states is used in place of the real expected distance. However, as we update the estimation of the distance between states based on estimating the distance of successor states, such bootstrapping steps may introduce bias and reduce the representations' consistency. Therefore, we adopt the momentum encoder~\cite{DBLP:conf/cvpr/He0WXG20} to stabilize the representation of the states.

Specifically, we draw batch of state pairs and minimise the mean square error (MSE) between both sides of SimSR to guide the learning of the encoder:

\begin{equation}
\begin{aligned}
    \mathcal{L}(\phi)=&\mathbb{E}_{(\mathbf{x},r(\mathbf{x},\mathbf{a}),\mathbf{a},\mathbf{x}'),(\mathbf{y},r(\mathbf{x},\mathbf{a}), \mathbf{a},\mathbf{y}')\sim \mathcal{D}}\big(\\&\overline{\text{cos}}(\phi(\mathbf{x}),\phi(\mathbf{y}))-\text{Target}\big)^2
    \label{eq:mse_loss}
\end{aligned}
\end{equation}
where $\mathcal{D}$ is the replay buffer and:
\begin{equation}
\begin{aligned}
    \text{Target}=\begin{cases}&|r_{\mathbf{x}}^{\pi}-r_{\mathbf{y}}^{\pi}|+\gamma \overline{\text{cos}}(\hat{\phi}(\mathbf{x}'),\hat{\phi}(\mathbf{y}')), \\& \text {if applying Theorem 2.} \\ &|r_{\mathbf{x}}^{\pi}-r_{\mathbf{y}}^{\pi}|+\gamma \overline{\text{cos}}(\mathbf{s}',\mathbf{u}')
    \\& \text { if applying Theorem 3.}\end{cases}
\end{aligned}
\label{eq:momentum_two_types}
\end{equation}

In Eq.~\ref{eq:momentum_two_types}, $\mathbf{s}'\sim\mathcal{P}_{i}(\cdot|\hat{\phi}(\mathbf{x}),\mathbf{a})$, $\mathbf{u}'\sim\mathcal{P}_{i}(\cdot|\hat{\phi}(\mathbf{y}),\mathbf{a})$,  $i$ is sampled from uniform distribution $U(1,K)$, and the momentum encoder $\hat{\phi}$ is the exponential moving average (EMA) of the encoder $\phi$:
\begin{equation}
    \hat{\phi} \leftarrow m \hat{\phi}+(1-m) \phi,
\end{equation}
where $m\in[0,1)$ is the momentum coefficient.

\begin{figure*}[h]
\centering
\subfigure[Cartpole Swingup]{
\includegraphics[width=0.18\textwidth]{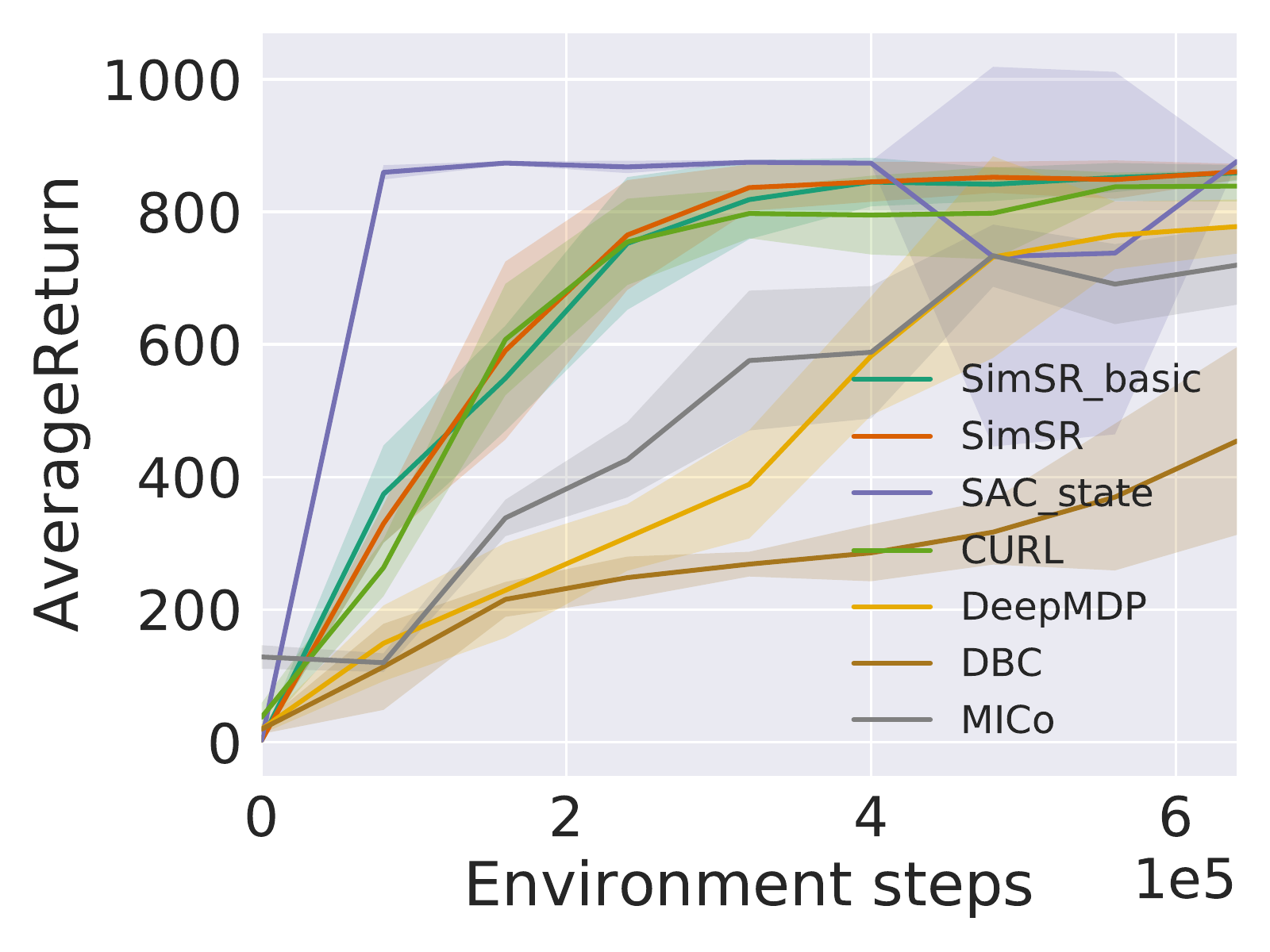}
}
\subfigure[Cheetah Run]{
\includegraphics[width=0.18\textwidth]{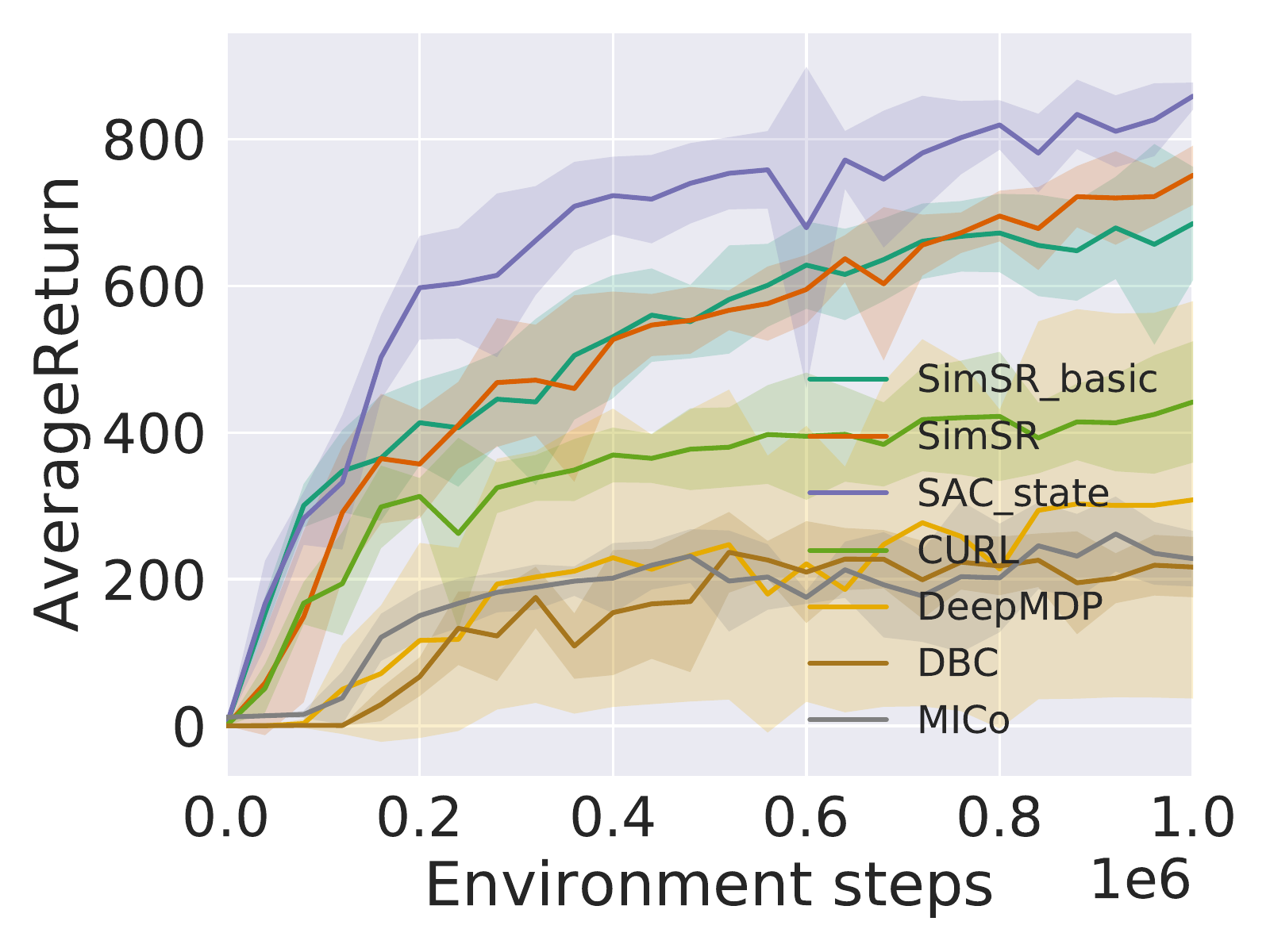}
}
\subfigure[Hopper Hop]{
\includegraphics[width=0.18\textwidth]{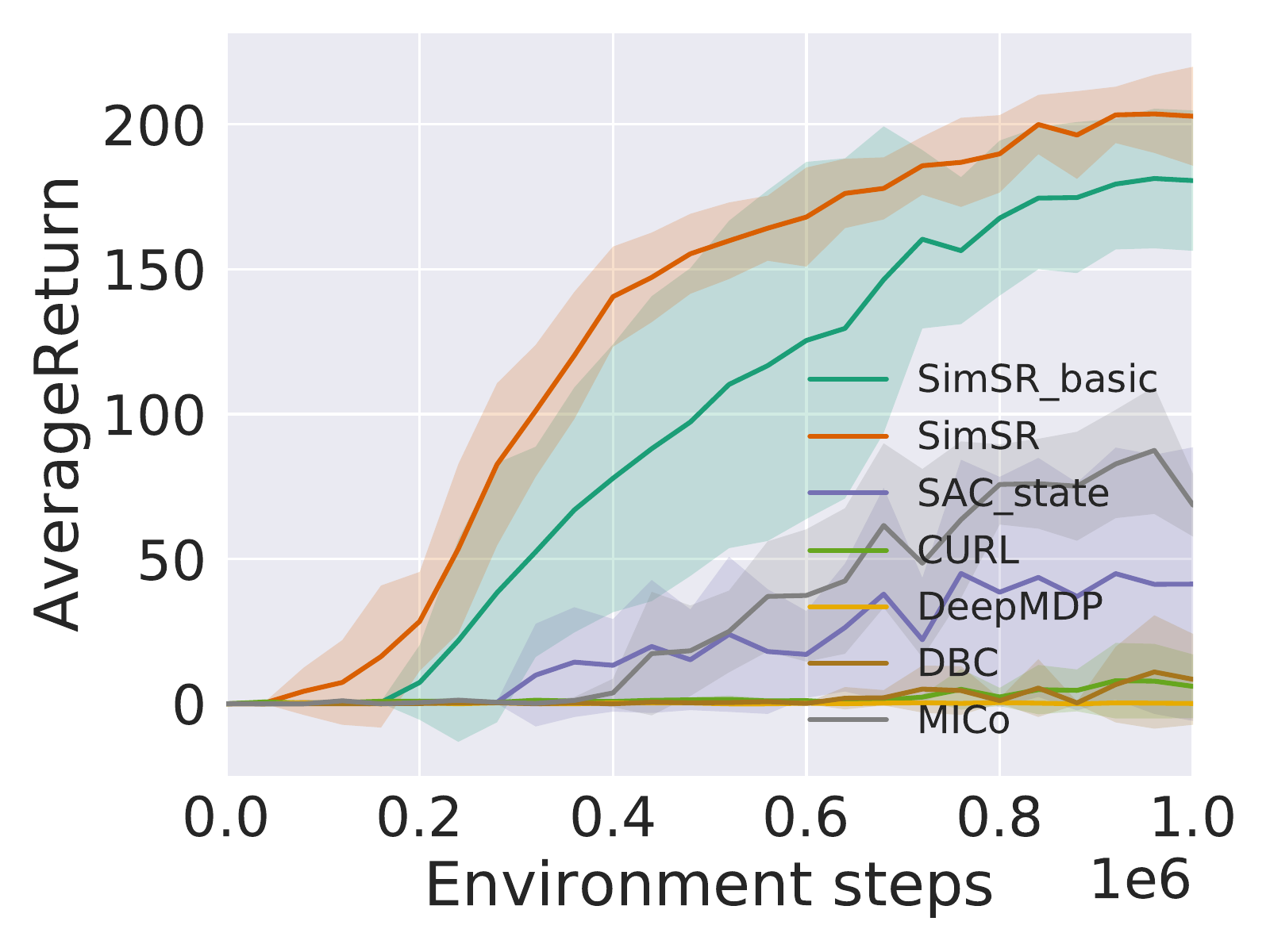}
}
\subfigure[Hopper Stand]{
\includegraphics[width=0.18\textwidth]{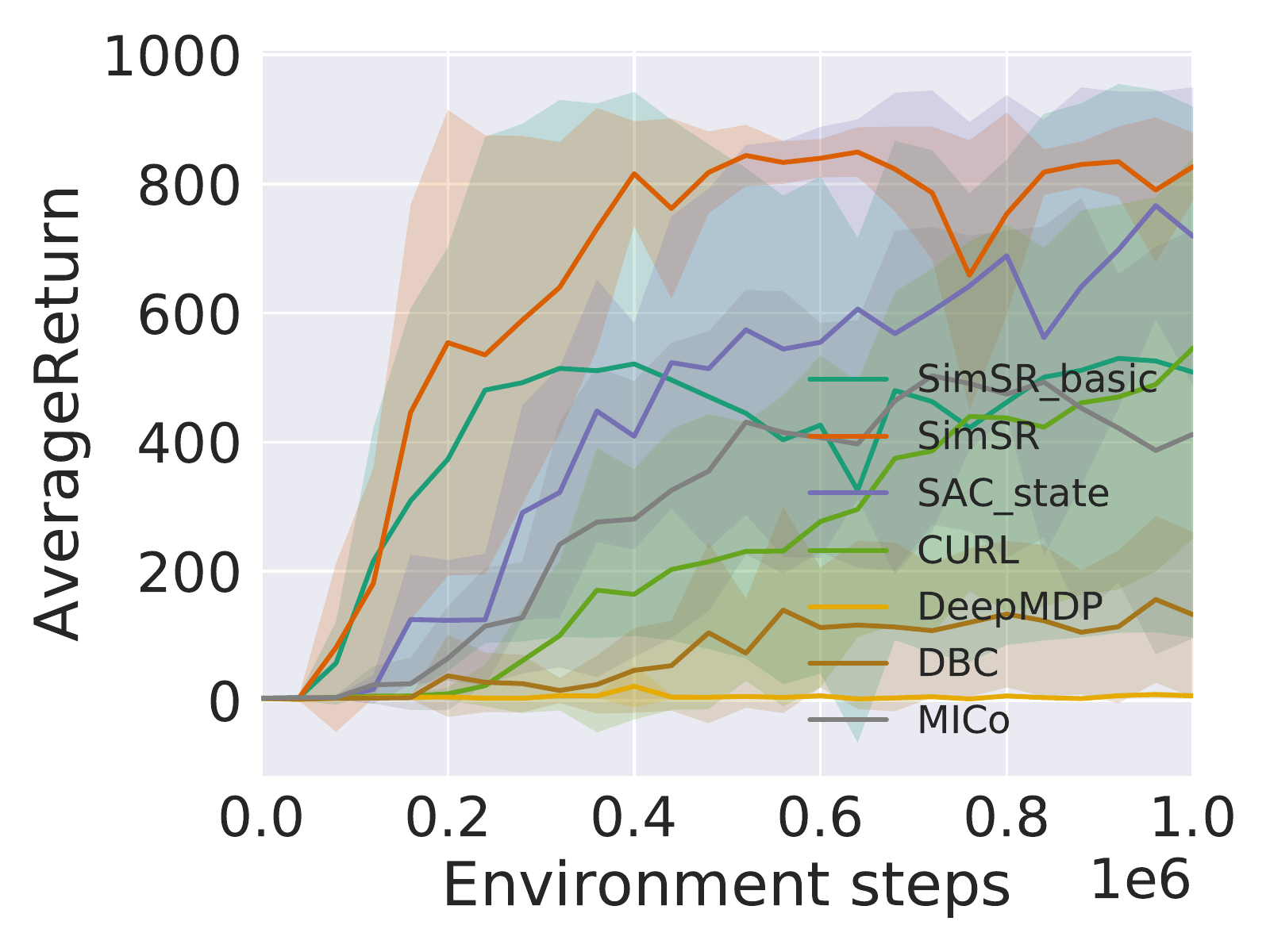}
}
\subfigure[Finger Spin]{
\includegraphics[width=0.18\textwidth]{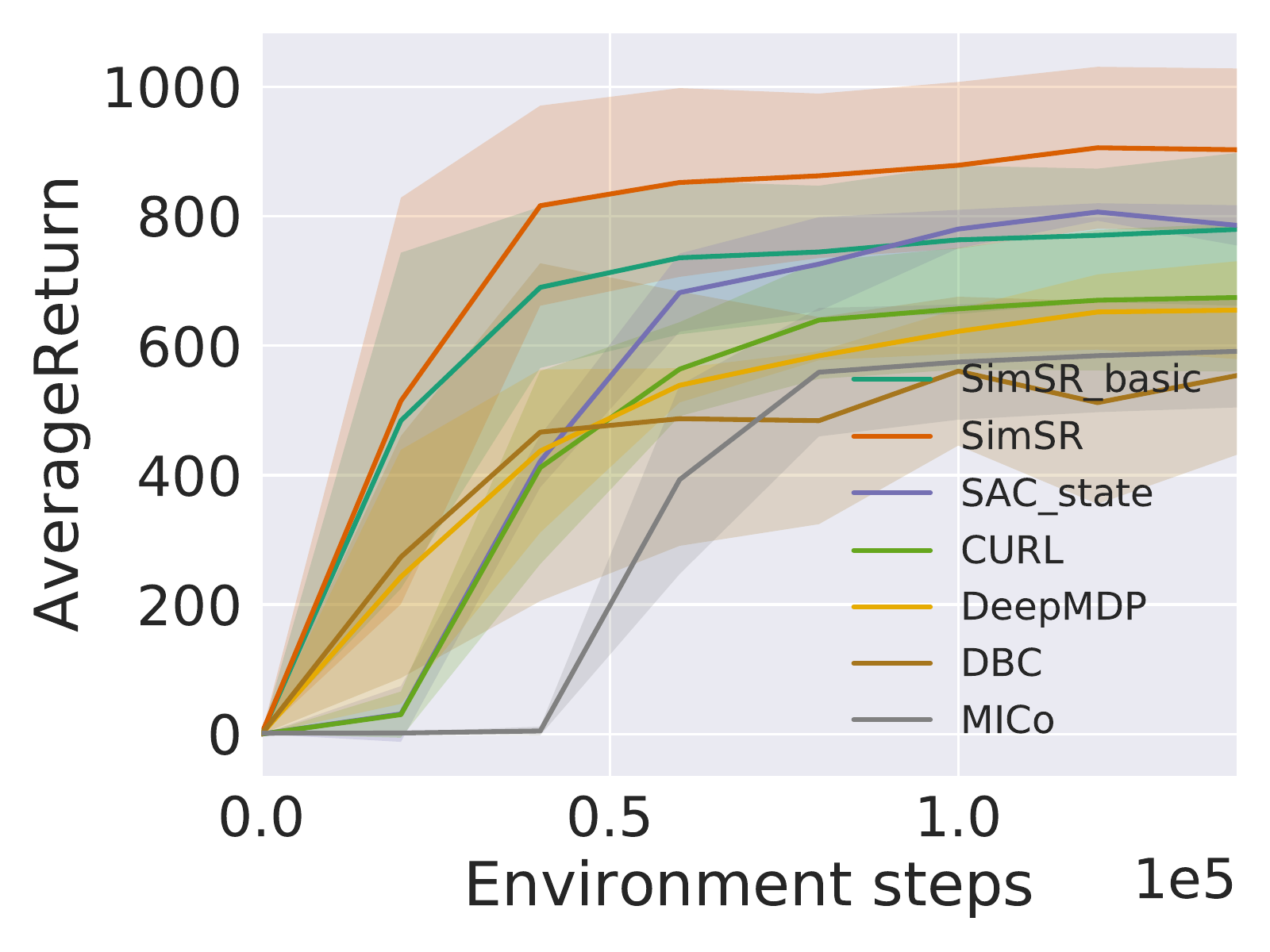}
}
\caption{Performance comparison on 5 DMC tasks over 5 seeds with one standard error shaded in the default setting. For each seed, the average return is computed every 10,000 training steps, averaging over 10 episodes. The horizontal axis indicates the number of environment steps. The vertical axis indicates the average return. }
\label{fig/standard_curve}
\end{figure*}

Technically, as both $\ell_1$ distance (the reward difference) and cosine distance (next state difference) can be easily extended to matrix operation, we can measure the cross-correlation matrix between the batch-wise state representations and optimize it to accelerate the representation learning. With Theorem~\ref{theorem:SimSR_op}, the loss function can be rewritten as:
\begin{equation}
\label{eq:batchwise_mse_loss}
\begin{aligned}
    \mathbb{E}_{B_j}[\mathcal{L}(\phi)] = &
    \mathbb{E}_{B_j}\bigg[(1-\frac{\phi(X)^T\cdot\phi(X)}{\|\phi(X)\|\cdot\|\phi(X)\|}) - \\&(|R^T-R|+\gamma (1-\frac{\hat{\phi}(X')^T\cdot\hat{\phi}(X')}{\|\hat{\phi}(X')\|\cdot\|\hat{\phi}(X')\|}))\bigg]^2,
\end{aligned}
\end{equation}
where $B_j$ is a sample batch, $X$, $X'$, and $R$ are the corresponding batch-wise observations, batch-wise next observations, and batch-wise rewards, respectively. In contrast, DBC~\cite{DBLP:conf/iclr/0001MCGL21} develops its loss function at sample-wise level, resulting in lower computational efficiency. 

The proposed representation learning approach can be viewed as an auxiliary task with SimSR objective to learn an encoder, on top of which the policy and value network can be easily built. Meanwhile, since the encoder, the policy network, and the value network can be trained simultaneously, our approach can be easily adopted to any model-based or model-free RL methods as an add-on.
In practice, we build agents by combining our approach with soft actor critic (SAC) algorithm~\cite{DBLP:journals/corr/abs-1812-05905} to devise a practical reinforcement learning method. The training process is illustrated by Algorithm~\ref{alg:simsr}. 
\renewcommand\algorithmiccomment[1]{\hfill $\triangleright$ #1}

\begin{algorithm}[H]
    \small
    \begin{algorithmic}[1]
    \floatname{algorithm}{Procedure}
    \FOR {Time $t = 0$ to $\infty$}
    \STATE{Encode state $\textbf{s}_t = \phi(\textbf{x}_{t})$}
    \STATE{Execute action $\textbf{a}_t \sim \pi(\textbf{s}_t)$}
    \STATE{Record data: $\mathcal{D} \leftarrow \mathcal{D} \cup \{\textbf{x}_{t}, \textbf{a}_t, \textbf{x}_{t+1}, r_{t+1}\}$}
    \STATE{Sample batch $B_i\sim\mathcal{D}$}
    \STATE{Train encoder: $\mathbb{E}_{B_i}[\mathcal{L}(\phi)]$ \algorithmiccomment{Eq.~\ref{eq:batchwise_mse_loss}}}
    \STATE{Train dynamics: $\mathbb{E}_{B_i}[\mathcal{L}_{\mathcal{P}}(\theta_j)]$, where $j\in\{1,...,K\}$ \algorithmiccomment{Eq.~\ref{eq:dynamics_loss}}}
    \STATE{Train value function: $\mathbb{E}_{B_i}[\mathcal{L}(V)]$ \algorithmiccomment{SAC algorithm}}
    \STATE{Train policy function: $\mathbb{E}_{B_i}[\mathcal{L}(\pi)]$ \algorithmiccomment{SAC algorithm}}
    \ENDFOR
    \end{algorithmic}
    \caption{SimSR algorithm}
    \label{alg:simsr}
\end{algorithm}

We notice that in some tasks, the critic loss explodes after few iterations and the value function diverges. This happens when agents start learning new behaviors and experiencing new states where the scale of MSE loss may increase exponentially large. Since both representation learning and reinforcement learning use the MSE loss, the combination of them aggravates the volatility, thus causing the explosion. Therefore, we use Huber loss to alleviate the explosion problem.

The most related methods to our work are DBC~\cite{DBLP:conf/iclr/0001MCGL21}, MICo~\cite{castro2021mico}, and DeepMDP~\cite{DBLP:conf/icml/GeladaKBNB19}. Table~\ref{tab/compare} provides the comparison of their key features. Detailed comparison with previous work is provided in Appendix.

\begin{table}
\caption{Comparison of key features in different algorithms. ``Dis. consistency'' means that the distance computed in the representation space is consistent with the base ``metric'' that is used in behavioral difference. ``-'' means ``does not apply''.}
\label{tab/compare}
\centering
\scriptsize
\begin{tabular}{ c c c c c c} 
 \toprule
    & Dis. consist & Unit length & Zero self-dis. &  Learn dynamics \\
 \midrule
SimSR & \cmark & \cmark &  \cmark &  \cmark \\
DBC & \xmark & \xmark & \cmark & \cmark \\
MICo & \cmark & \xmark & \xmark & \xmark \\
DeepMDP & - & \xmark& - & \cmark \\
 \bottomrule
\end{tabular}
\end{table}

\section{Experiments}
We experimented to investigate the following questions: (1) In comparison with state-of-the-art algorithms, does SimSR have a better performance in terms of sample efficiency? (2) Can SimSR learn robust state representation? (3) How is the generalization performance of the learned representation? 

Accordingly, we first evaluated our method in several standard control tasks from the DeepMind control (DMC) suite ~\cite{DBLP:journals/corr/abs-1801-00690}. We then evaluated the robustness of our method to test if it can handle more realistic and complicated scenarios, where we replaced the background of the environment with natural video as distractors. Finally, we tested the generalization performance of the state representation in unseen tasks.

\begin{figure}[htbp]
\centering
\subfigure{
\includegraphics[width=0.11\textwidth]{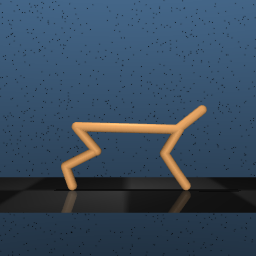}
}
\subfigure{
\includegraphics[width=0.11\textwidth]{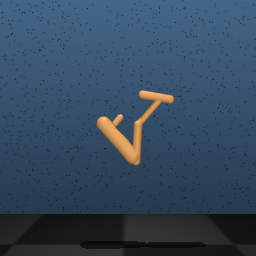}
}
\subfigure{
\includegraphics[width=0.11\textwidth]{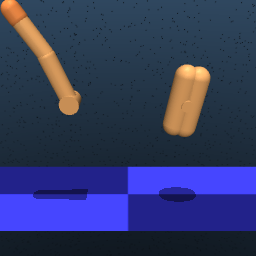}
}
\subfigure{
\includegraphics[width=0.11\textwidth]{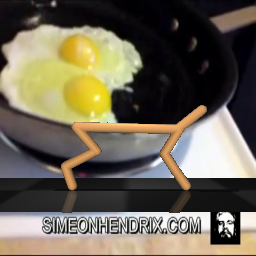}
}
\subfigure{
\includegraphics[width=0.11\textwidth]{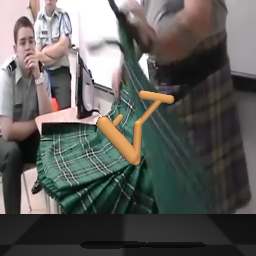}
}
\subfigure{
\includegraphics[width=0.11\textwidth]{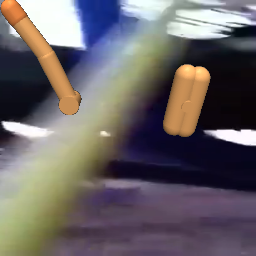}
}
\caption{Pixel observations in DMC in the default setting (top row) and in the natural video background settings (bottom row) for cheetah (left column), hopper (middle column), and finger (right column).
}
\label{fig/env_settings}
\end{figure}

\subsection{Effectiveness and robustness}
\subsubsection{Settings}
We compared the proposed approaches, SimSR\_basic (Theorem 2) and SimSR (Theorem 4), against the selected baselines. As one goal is to close the performance gap between the physical-state based methods and images based methods, we first compared with SAC\_state, which is non-visual SAC~\cite{DBLP:journals/corr/abs-1812-05905} accepting physical features (e.g., position, angle, height) as inputs. 
We then compared with CURL~\cite{DBLP:conf/icml/LaskinSA20}, which is a recent state-of-the-art image-based RL method that learns latent representation via data augmentation and contrastive learning. 
Thereafter, we compared with the approaches that are most relevant to ours, including DBC~\cite{DBLP:conf/iclr/0001MCGL21}, MICo~\cite{castro2021mico}, and DeepMDP~\cite{DBLP:conf/icml/GeladaKBNB19}\footnote{The results of DBC and DeepMDP are taken from \cite{DBLP:conf/iclr/0001MCGL21} with five seeds.  Our code is available at https://github.com/bit1029public/SimSR}. Due to the page limit and the absence of some results of DBC and DeepMDP, we provide more comprehensive performance comparisons on \texttt{Ball\_In\_Cup\_Catch}, \texttt{Pendulum\_Swingup}, and additional environments in Appendix. We keep all hyper-parameters of the algorithm fixed throughout experiments except the action repeat which follows the convention to ensure a fair comparison. The settings of all hyper-parameters and architectures are also provided in Appendix.

\subsubsection{Experiments with default setting}
As shown in the top row of Figure~\ref{fig/env_settings}, the default setting, which is provided by DMC, has simple backgrounds for the pixel observations.  Figure~\ref{fig/standard_curve} demonstrates that SimSR outperforms all selected state-of-the-art methods by a large margin on 3 of 5 tasks (\texttt{Finger\_Spin}, \texttt{Hopper\_Hop}, \texttt{Hopper\_Stand}) and remains competitive in the rest of 2 tasks. As discussed, MICo may encounter the representation collapse problem where different observations are encoded to the same embedding. Since the policy network and the value network are built on top of the encoder, they both may collapse and result in inaccurate learning in MICo. 

\begin{figure}[tbp]
\centering
\subfigure[Cheetah Run]{
\includegraphics[width=0.14\textwidth]{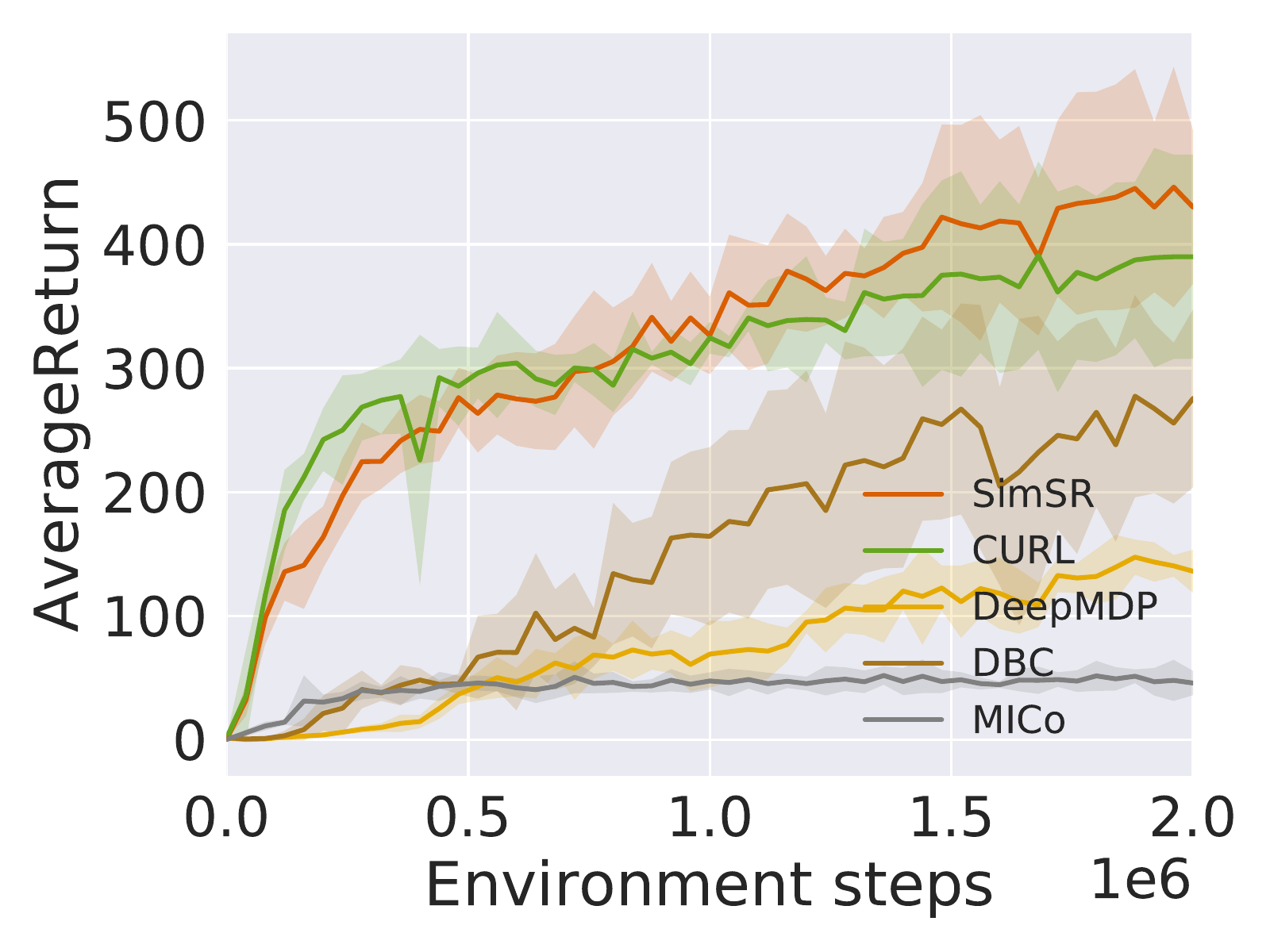}
}
\subfigure[Hopper Stand]{
\includegraphics[width=0.14\textwidth]{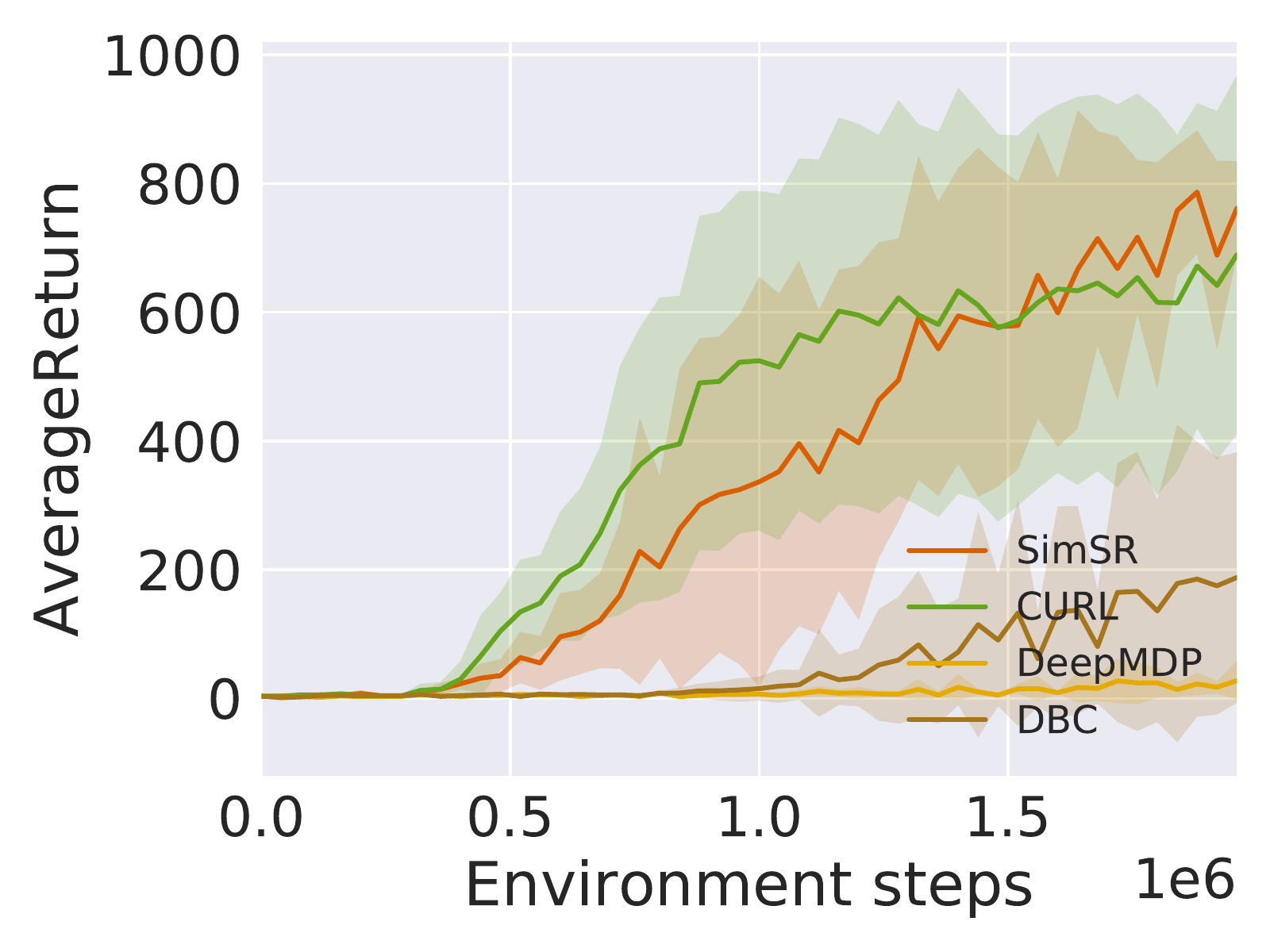}
}
\subfigure[Finger Spin]{
\includegraphics[width=0.14\textwidth]{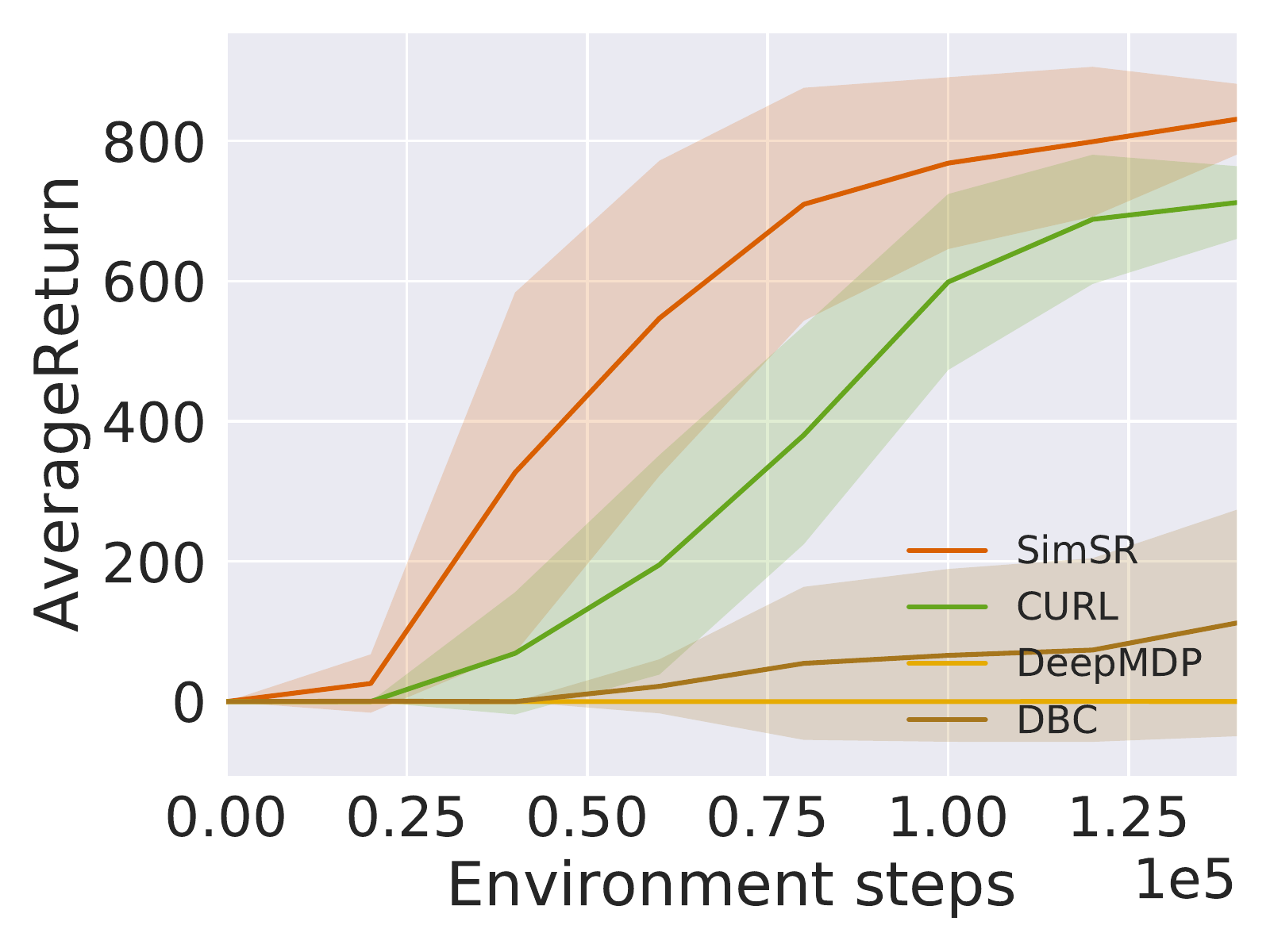}
}
\caption{Performance comparisons on 3 DMC tasks over 5 seeds with one standard error shaded in the natural video setting.}
\label{fig:natural}
\end{figure}

The results prove that SimSR can generally close the gap between image observation based RL methods and physical state based RL methods regarding the downstream performance without any data augmentation techniques. Moreover, in comparison with the approaches with bisimulation metric, SimSR consistently outperforms MICo, DBC, and DeepMDP on all tasks, showing the effectiveness of the proposed SimSR distance. SimSR performs better than SimSR\_basic in most cases, confirming that learning latent transitions can provide agents more diverse knowledge of the goals and achieve better sample efficiency.

\subsubsection{Experiments with natural video setting} 
To investigate the robustness of representation learning against task-irrelevant information, we used natural videos from the Kinetics dataset~\cite{DBLP:journals/corr/KayCSZHVVGBNSZ17}, as shown in the bottom row of Figure~\ref{fig/env_settings}, to substitute the default simple background in the experiments. Figure~\ref{fig:natural} demonstrates that SimSR is capable of filtering out task-irrelevant information, thus achieving better performance against other bisimulation-based methods.

Since CURL combines contrastive learning with data augmentation, its contrastive loss incorporates the prior knowledge of the environment into the encoder. Therefore, CURL is able to distinguish task irrelevant signals in visual spaces, thus having better performance than other methods without encoding prior knowledge. Still, SimSR outperforms CURL.

\subsection{Generalization experiments}

To evaluate the generalization capability, we tested the learned representation on different tasks, including \texttt{Walker\_Walk}, \texttt{Walker\_Stand}, and \texttt{Walker\_Run}, which share the same observation space but have different reward functions.

We first trained SimSR on \texttt{Walker\_Walk} until convergence to learn the base encoder. To test its generalization capability, We then trained two SAC agents with image-based inputs on \texttt{Walker\_Stand} and \texttt{Walker\_Run}, respectively. One SAC agent utilized the base encoder learned in \texttt{Walker\_Walk} without further updates during the training process, while the other SAC agent learned its own encoder from scratch. For comparison, we also trained SimSR agents in both tasks to represent the specialized encoders in each task. Figure~\ref{fig/generalization} illustrates the learning curves of all agents in both tasks.

\begin{figure}[t]
\centering
\subfigure[Walker Stand]{
\includegraphics[width=0.21\textwidth]{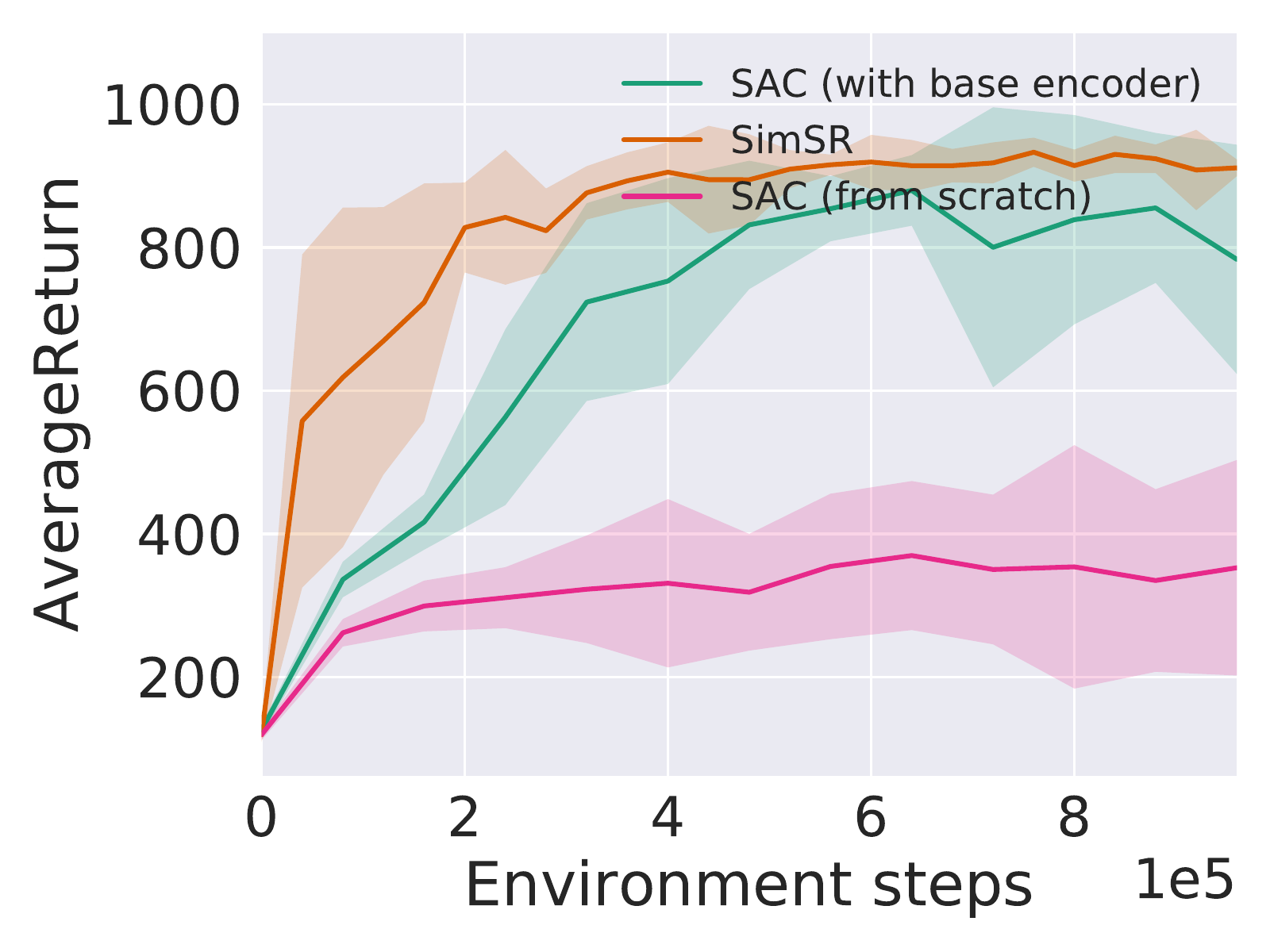}
}
\subfigure[Walker Run]{
\includegraphics[width=0.21\textwidth]{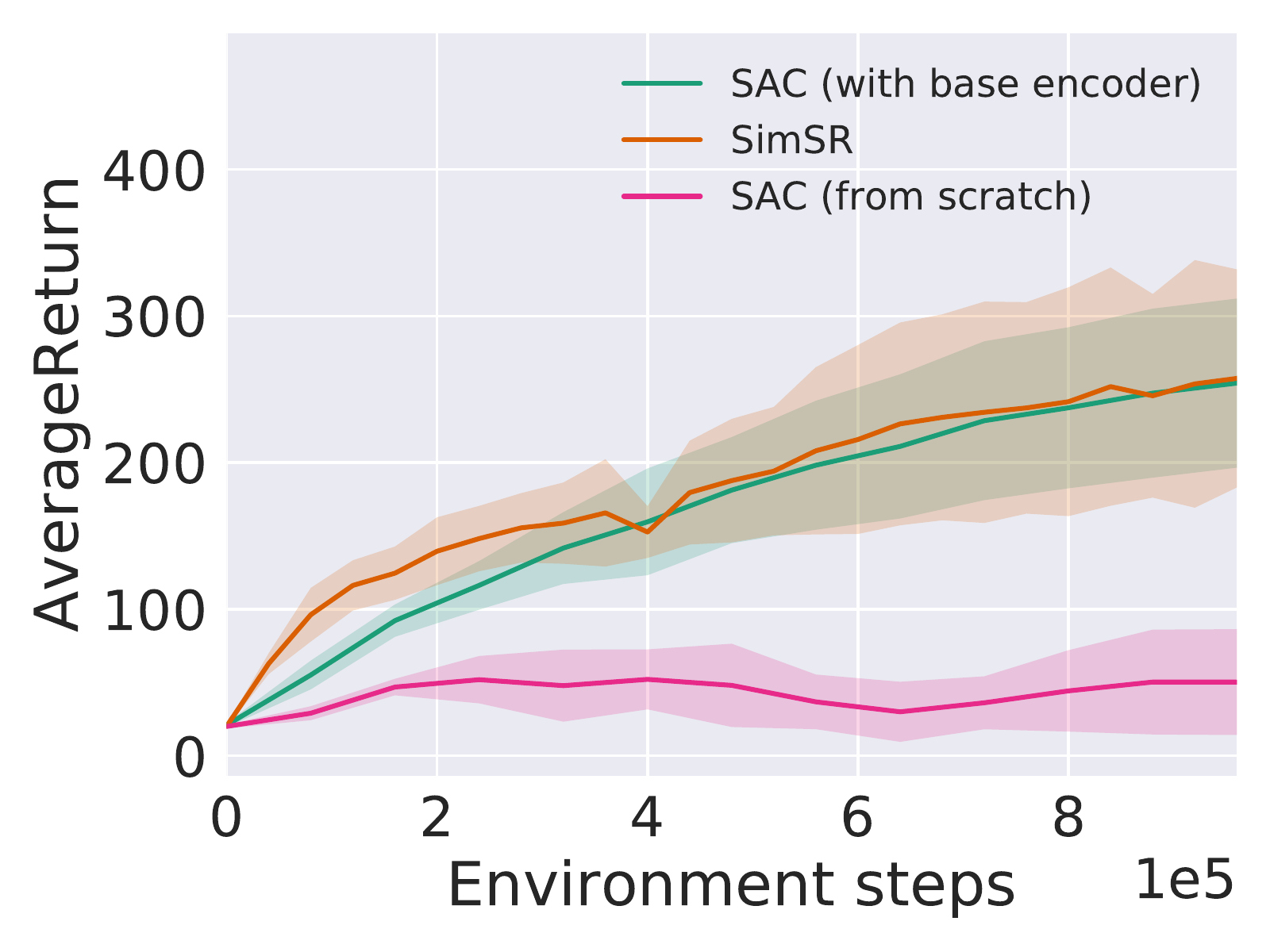}
}
\caption{The comparison between SimSR agent and two types of SAC agents on \texttt{Walker Stand} and \texttt{Walker Run} tasks. The tasks share similar image observations, but have different reward.}
\label{fig/generalization}
\end{figure}

The results demonstrate that SimSR agents perform the best in both tasks as they are fine-tuned to their specific environments. The use of learned representations in generalized tasks, as represented by the SAC agent with the learned encoder, achieves very close performance next to the ideal cases of specific models. Moreover, the SAC agent with the learned encoder significantly outperforms the SAC agent that learned from scratch, demonstrating the generalization capability of SimSR. The results empirically prove the potential of SimSR in learning latent representations which can be generalized well over the tasks that share the same observation space but with different reward functions and goals.

\subsection{Ablation study}
We investigated the importance of $\ell_2$ normalization, which guarantees the properties of unit length and zero self-distance in the results. Without $\ell_2$ normalization, the agent cannot produce acceptable performance. This finding empirically underpins our theoretical analysis: if the distance is not unit length or zero self-distance, the representation may easily collapse.
Besides, we also analyzed the effectiveness of the ensemble transition models. 
The results shows that even with a single probabilistic model, the agent can still achieve remarkable performance, possibly with higher variances. This confirms our hypothesis that the ensemble model can estimate the uncertainty better, thus leading to more stable representations. All results are presented in Appendix.

\section{Conclusion and Future Work}
In the paper, we presented SimSR, a distance-based representation learning method, which can be easily adopted to any model-based or model-free RL methods. 
The experiments demonstrated that SimSR is capable of learning robust and generalizable state representation and further enhancing policy learning. The results show that we improved the performance considerably over existing image-based RL methods.

Outside the line of this research, learning state representation with data augmentation can also improve the effectiveness of extracting task-relevant signals. One important future work is to combine two lines of approaches. 

\section{Acknowledgments}
This work has been partially supported by NSFC under Grant (U19B2020 and 61772074) and National Key R$\&$D Program of China under Grant (2017YFB0803300).

\bibliography{aaai22}

\begin{thebibliography}{36}
\providecommand{\natexlab}[1]{#1}

\bibitem[{Allen et~al.(2021)Allen, Parikh, Gottesman, and
  Konidaris}]{DBLP:journals/corr/abs-2106-04379}
Allen, C.; Parikh, N.; Gottesman, O.; and Konidaris, G. 2021.
\newblock Learning Markov State Abstractions for Deep Reinforcement Learning.
\newblock \emph{CoRR}, abs/2106.04379.

\bibitem[{Anand et~al.(2019)Anand, Racah, Ozair, Bengio, C{\^{o}}t{\'{e}}, and
  Hjelm}]{DBLP:conf/nips/AnandROBCH19}
Anand, A.; Racah, E.; Ozair, S.; Bengio, Y.; C{\^{o}}t{\'{e}}, M.; and Hjelm,
  R.~D. 2019.
\newblock Unsupervised State Representation Learning in Atari.
\newblock In Wallach, H.~M.; Larochelle, H.; Beygelzimer, A.;
  d'Alch{\'{e}}{-}Buc, F.; Fox, E.~B.; and Garnett, R., eds., \emph{Advances in
  Neural Information Processing Systems 32: Annual Conference on Neural
  Information Processing Systems 2019, NeurIPS 2019, December 8-14, 2019,
  Vancouver, BC, Canada}, 8766--8779.

\bibitem[{Bojanowski and Joulin(2017)}]{DBLP:conf/icml/BojanowskiJ17}
Bojanowski, P.; and Joulin, A. 2017.
\newblock Unsupervised Learning by Predicting Noise.
\newblock In Precup, D.; and Teh, Y.~W., eds., \emph{Proceedings of the 34th
  International Conference on Machine Learning, {ICML} 2017, Sydney, NSW,
  Australia, 6-11 August 2017}, volume~70 of \emph{Proceedings of Machine
  Learning Research}, 517--526. {PMLR}.

\bibitem[{Castro(2020)}]{DBLP:conf/aaai/Castro20}
Castro, P.~S. 2020.
\newblock Scalable Methods for Computing State Similarity in Deterministic
  Markov Decision Processes.
\newblock In \emph{The Thirty-Fourth {AAAI} Conference on Artificial
  Intelligence, {AAAI} 2020, The Thirty-Second Innovative Applications of
  Artificial Intelligence Conference, {IAAI} 2020, The Tenth {AAAI} Symposium
  on Educational Advances in Artificial Intelligence, {EAAI} 2020, New York,
  NY, USA, February 7-12, 2020}, 10069--10076. {AAAI} Press.

\bibitem[{Castro et~al.(2021)Castro, Kastner, Panangaden, and
  Rowland}]{castro2021mico}
Castro, P.~S.; Kastner, T.; Panangaden, P.; and Rowland, M. 2021.
\newblock {MIC}o: Improved representations via sampling-based state similarity
  for Markov decision processes.
\newblock In Beygelzimer, A.; Dauphin, Y.; Liang, P.; and Vaughan, J.~W., eds.,
  \emph{Advances in Neural Information Processing Systems}.

\bibitem[{Du et~al.(2019)Du, Krishnamurthy, Jiang, Agarwal, Dud{\'{\i}}k, and
  Langford}]{DBLP:conf/icml/DuKJAD019}
Du, S.~S.; Krishnamurthy, A.; Jiang, N.; Agarwal, A.; Dud{\'{\i}}k, M.; and
  Langford, J. 2019.
\newblock Provably efficient {RL} with Rich Observations via Latent State
  Decoding.
\newblock In Chaudhuri, K.; and Salakhutdinov, R., eds., \emph{Proceedings of
  the 36th International Conference on Machine Learning, {ICML} 2019, 9-15 June
  2019, Long Beach, California, {USA}}, volume~97 of \emph{Proceedings of
  Machine Learning Research}, 1665--1674. {PMLR}.

\bibitem[{Fan and Li(2021)}]{DBLP:journals/corr/abs-2102-13268}
Fan, J.; and Li, W. 2021.
\newblock Robust Deep Reinforcement Learning via Multi-View Information
  Bottleneck.
\newblock \emph{CoRR}, abs/2102.13268.

\bibitem[{Ferns, Panangaden, and Precup(2004)}]{DBLP:conf/uai/FernsPP04}
Ferns, N.; Panangaden, P.; and Precup, D. 2004.
\newblock Metrics for Finite Markov Decision Processes.
\newblock In Chickering, D.~M.; and Halpern, J.~Y., eds., \emph{{UAI} '04,
  Proceedings of the 20th Conference in Uncertainty in Artificial Intelligence,
  Banff, Canada, July 7-11, 2004}, 162--169. {AUAI} Press.

\bibitem[{Ferns and Precup(2014)}]{DBLP:conf/uai/FernsP14}
Ferns, N.; and Precup, D. 2014.
\newblock Bisimulation Metrics are Optimal Value Functions.
\newblock In Zhang, N.~L.; and Tian, J., eds., \emph{Proceedings of the
  Thirtieth Conference on Uncertainty in Artificial Intelligence, {UAI} 2014,
  Quebec City, Quebec, Canada, July 23-27, 2014}, 210--219. {AUAI} Press.

\bibitem[{Gelada et~al.(2019)Gelada, Kumar, Buckman, Nachum, and
  Bellemare}]{DBLP:conf/icml/GeladaKBNB19}
Gelada, C.; Kumar, S.; Buckman, J.; Nachum, O.; and Bellemare, M.~G. 2019.
\newblock DeepMDP: Learning Continuous Latent Space Models for Representation
  Learning.
\newblock In Chaudhuri, K.; and Salakhutdinov, R., eds., \emph{Proceedings of
  the 36th International Conference on Machine Learning, {ICML} 2019, 9-15 June
  2019, Long Beach, California, {USA}}, volume~97 of \emph{Proceedings of
  Machine Learning Research}, 2170--2179. {PMLR}.

\bibitem[{Givan, Dean, and Greig(2003)}]{DBLP:journals/ai/GivanDG03}
Givan, R.; Dean, T.~L.; and Greig, M. 2003.
\newblock Equivalence notions and model minimization in Markov decision
  processes.
\newblock \emph{Artif. Intell.}, 147(1-2): 163--223.

\bibitem[{Haarnoja et~al.(2018)Haarnoja, Zhou, Hartikainen, Tucker, Ha, Tan,
  Kumar, Zhu, Gupta, Abbeel, and Levine}]{DBLP:journals/corr/abs-1812-05905}
Haarnoja, T.; Zhou, A.; Hartikainen, K.; Tucker, G.; Ha, S.; Tan, J.; Kumar,
  V.; Zhu, H.; Gupta, A.; Abbeel, P.; and Levine, S. 2018.
\newblock Soft Actor-Critic Algorithms and Applications.
\newblock \emph{CoRR}, abs/1812.05905.

\bibitem[{Hafner et~al.(2019)Hafner, Lillicrap, Fischer, Villegas, Ha, Lee, and
  Davidson}]{DBLP:conf/icml/HafnerLFVHLD19}
Hafner, D.; Lillicrap, T.~P.; Fischer, I.; Villegas, R.; Ha, D.; Lee, H.; and
  Davidson, J. 2019.
\newblock Learning Latent Dynamics for Planning from Pixels.
\newblock In Chaudhuri, K.; and Salakhutdinov, R., eds., \emph{Proceedings of
  the 36th International Conference on Machine Learning, {ICML} 2019, 9-15 June
  2019, Long Beach, California, {USA}}, volume~97 of \emph{Proceedings of
  Machine Learning Research}, 2555--2565. {PMLR}.

\bibitem[{He et~al.(2020)He, Fan, Wu, Xie, and
  Girshick}]{DBLP:conf/cvpr/He0WXG20}
He, K.; Fan, H.; Wu, Y.; Xie, S.; and Girshick, R.~B. 2020.
\newblock Momentum Contrast for Unsupervised Visual Representation Learning.
\newblock In \emph{2020 {IEEE/CVF} Conference on Computer Vision and Pattern
  Recognition, {CVPR} 2020, Seattle, WA, USA, June 13-19, 2020}, 9726--9735.
  {IEEE}.

\bibitem[{Jaderberg et~al.(2017)Jaderberg, Mnih, Czarnecki, Schaul, Leibo,
  Silver, and Kavukcuoglu}]{DBLP:conf/iclr/JaderbergMCSLSK17}
Jaderberg, M.; Mnih, V.; Czarnecki, W.~M.; Schaul, T.; Leibo, J.~Z.; Silver,
  D.; and Kavukcuoglu, K. 2017.
\newblock Reinforcement Learning with Unsupervised Auxiliary Tasks.
\newblock In \emph{5th International Conference on Learning Representations,
  {ICLR} 2017, Toulon, France, April 24-26, 2017, Conference Track
  Proceedings}. OpenReview.net.

\bibitem[{Kay et~al.(2017)Kay, Carreira, Simonyan, Zhang, Hillier,
  Vijayanarasimhan, Viola, Green, Back, Natsev, Suleyman, and
  Zisserman}]{DBLP:journals/corr/KayCSZHVVGBNSZ17}
Kay, W.; Carreira, J.; Simonyan, K.; Zhang, B.; Hillier, C.; Vijayanarasimhan,
  S.; Viola, F.; Green, T.; Back, T.; Natsev, P.; Suleyman, M.; and Zisserman,
  A. 2017.
\newblock The Kinetics Human Action Video Dataset.
\newblock \emph{CoRR}, abs/1705.06950.

\bibitem[{Lange and Riedmiller(2010)}]{DBLP:conf/ijcnn/LangeR10}
Lange, S.; and Riedmiller, M.~A. 2010.
\newblock Deep auto-encoder neural networks in reinforcement learning.
\newblock In \emph{International Joint Conference on Neural Networks, {IJCNN}
  2010, Barcelona, Spain, 18-23 July, 2010}, 1--8. {IEEE}.

\bibitem[{Lange, Riedmiller, and
  Voigtl{\"{a}}nder(2012)}]{DBLP:conf/ijcnn/LangeRV12}
Lange, S.; Riedmiller, M.~A.; and Voigtl{\"{a}}nder, A. 2012.
\newblock Autonomous reinforcement learning on raw visual input data in a real
  world application.
\newblock In \emph{The 2012 International Joint Conference on Neural Networks
  (IJCNN), Brisbane, Australia, June 10-15, 2012}, 1--8. {IEEE}.

\bibitem[{Larsen and Skou(1989)}]{DBLP:conf/popl/LarsenS89}
Larsen, K.~G.; and Skou, A. 1989.
\newblock Bisimulation Through Probabilistic Testing.
\newblock In \emph{Conference Record of the Sixteenth Annual {ACM} Symposium on
  Principles of Programming Languages, Austin, Texas, USA, January 11-13,
  1989}, 344--352. {ACM} Press.

\bibitem[{Laskin, Srinivas, and Abbeel(2020)}]{DBLP:conf/icml/LaskinSA20}
Laskin, M.; Srinivas, A.; and Abbeel, P. 2020.
\newblock {CURL:} Contrastive Unsupervised Representations for Reinforcement
  Learning.
\newblock In \emph{Proceedings of the 37th International Conference on Machine
  Learning, {ICML} 2020, 13-18 July 2020, Virtual Event}, volume 119 of
  \emph{Proceedings of Machine Learning Research}, 5639--5650. {PMLR}.

\bibitem[{Lee et~al.(2020{\natexlab{a}})Lee, Nagabandi, Abbeel, and
  Levine}]{DBLP:conf/nips/LeeNAL20}
Lee, A.~X.; Nagabandi, A.; Abbeel, P.; and Levine, S. 2020{\natexlab{a}}.
\newblock Stochastic Latent Actor-Critic: Deep Reinforcement Learning with a
  Latent Variable Model.
\newblock In Larochelle, H.; Ranzato, M.; Hadsell, R.; Balcan, M.; and Lin, H.,
  eds., \emph{Advances in Neural Information Processing Systems 33: Annual
  Conference on Neural Information Processing Systems 2020, NeurIPS 2020,
  December 6-12, 2020, virtual}.

\bibitem[{Lee et~al.(2020{\natexlab{b}})Lee, Fischer, Liu, Guo, Lee, Canny, and
  Guadarrama}]{DBLP:conf/nips/LeeFLGLCG20}
Lee, K.; Fischer, I.; Liu, A.; Guo, Y.; Lee, H.; Canny, J.; and Guadarrama, S.
  2020{\natexlab{b}}.
\newblock Predictive Information Accelerates Learning in {RL}.
\newblock In Larochelle, H.; Ranzato, M.; Hadsell, R.; Balcan, M.; and Lin, H.,
  eds., \emph{Advances in Neural Information Processing Systems 33: Annual
  Conference on Neural Information Processing Systems 2020, NeurIPS 2020,
  December 6-12, 2020, virtual}.

\bibitem[{Li, Walsh, and Littman(2006)}]{DBLP:conf/isaim/LiWL06}
Li, L.; Walsh, T.~J.; and Littman, M.~L. 2006.
\newblock Towards a Unified Theory of State Abstraction for MDPs.
\newblock In \emph{International Symposium on Artificial Intelligence and
  Mathematics, {ISAIM} 2006, Fort Lauderdale, Florida, USA, January 4-6, 2006}.

\bibitem[{Pathak et~al.(2017)Pathak, Agrawal, Efros, and
  Darrell}]{DBLP:conf/icml/PathakAED17}
Pathak, D.; Agrawal, P.; Efros, A.~A.; and Darrell, T. 2017.
\newblock Curiosity-driven Exploration by Self-supervised Prediction.
\newblock In Precup, D.; and Teh, Y.~W., eds., \emph{Proceedings of the 34th
  International Conference on Machine Learning, {ICML} 2017, Sydney, NSW,
  Australia, 6-11 August 2017}, volume~70 of \emph{Proceedings of Machine
  Learning Research}, 2778--2787. {PMLR}.

\bibitem[{Shelhamer et~al.(2017)Shelhamer, Mahmoudieh, Argus, and
  Darrell}]{DBLP:conf/iclr/ShelhamerMAD17}
Shelhamer, E.; Mahmoudieh, P.; Argus, M.; and Darrell, T. 2017.
\newblock Loss is its own Reward: Self-Supervision for Reinforcement Learning.
\newblock In \emph{5th International Conference on Learning Representations,
  {ICLR} 2017, Toulon, France, April 24-26, 2017, Workshop Track Proceedings}.
  OpenReview.net.

\bibitem[{Stooke et~al.(2021)Stooke, Lee, Abbeel, and
  Laskin}]{DBLP:conf/icml/StookeLAL21}
Stooke, A.; Lee, K.; Abbeel, P.; and Laskin, M. 2021.
\newblock Decoupling Representation Learning from Reinforcement Learning.
\newblock In Meila, M.; and Zhang, T., eds., \emph{Proceedings of the 38th
  International Conference on Machine Learning, {ICML} 2021, 18-24 July 2021,
  Virtual Event}, volume 139 of \emph{Proceedings of Machine Learning
  Research}, 9870--9879. {PMLR}.

\bibitem[{Sutton and Barto(1998)}]{DBLP:books/lib/SuttonB98}
Sutton, R.~S.; and Barto, A.~G. 1998.
\newblock \emph{Reinforcement learning - an introduction}.
\newblock Adaptive computation and machine learning. {MIT} Press.
\newblock ISBN 978-0-262-19398-6.

\bibitem[{Tassa et~al.(2018)Tassa, Doron, Muldal, Erez, Li, de~Las~Casas,
  Budden, Abdolmaleki, Merel, Lefrancq, Lillicrap, and
  Riedmiller}]{DBLP:journals/corr/abs-1801-00690}
Tassa, Y.; Doron, Y.; Muldal, A.; Erez, T.; Li, Y.; de~Las~Casas, D.; Budden,
  D.; Abdolmaleki, A.; Merel, J.; Lefrancq, A.; Lillicrap, T.~P.; and
  Riedmiller, M.~A. 2018.
\newblock DeepMind Control Suite.
\newblock \emph{CoRR}, abs/1801.00690.

\bibitem[{van~den Oord, Li, and
  Vinyals(2018)}]{DBLP:journals/corr/abs-1807-03748}
van~den Oord, A.; Li, Y.; and Vinyals, O. 2018.
\newblock Representation Learning with Contrastive Predictive Coding.
\newblock \emph{CoRR}, abs/1807.03748.

\bibitem[{Wang et~al.(2017)Wang, Xiang, Cheng, and
  Yuille}]{DBLP:conf/mm/WangXCY17}
Wang, F.; Xiang, X.; Cheng, J.; and Yuille, A.~L. 2017.
\newblock NormFace: L\({}_{\mbox{2}}\) Hypersphere Embedding for Face
  Verification.
\newblock In Liu, Q.; Lienhart, R.; Wang, H.; Chen, S.~K.; Boll, S.; Chen,
  Y.~P.; Friedland, G.; Li, J.; and Yan, S., eds., \emph{Proceedings of the
  2017 {ACM} on Multimedia Conference, {MM} 2017, Mountain View, CA, USA,
  October 23-27, 2017}, 1041--1049. {ACM}.

\bibitem[{Wang and Isola(2020)}]{DBLP:conf/icml/0001I20}
Wang, T.; and Isola, P. 2020.
\newblock Understanding Contrastive Representation Learning through Alignment
  and Uniformity on the Hypersphere.
\newblock In \emph{Proceedings of the 37th International Conference on Machine
  Learning, {ICML} 2020, 13-18 July 2020, Virtual Event}, volume 119 of
  \emph{Proceedings of Machine Learning Research}, 9929--9939. {PMLR}.

\bibitem[{Warde{-}Farley et~al.(2019)Warde{-}Farley, de~Wiele, Kulkarni,
  Ionescu, Hansen, and Mnih}]{DBLP:conf/iclr/Warde-FarleyWKI19}
Warde{-}Farley, D.; de~Wiele, T.~V.; Kulkarni, T.~D.; Ionescu, C.; Hansen, S.;
  and Mnih, V. 2019.
\newblock Unsupervised Control Through Non-Parametric Discriminative Rewards.
\newblock In \emph{7th International Conference on Learning Representations,
  {ICLR} 2019, New Orleans, LA, USA, May 6-9, 2019}. OpenReview.net.

\bibitem[{Yarats et~al.(2021{\natexlab{a}})Yarats, Fergus, Lazaric, and
  Pinto}]{DBLP:conf/icml/YaratsFLP21}
Yarats, D.; Fergus, R.; Lazaric, A.; and Pinto, L. 2021{\natexlab{a}}.
\newblock Reinforcement Learning with Prototypical Representations.
\newblock In Meila, M.; and Zhang, T., eds., \emph{Proceedings of the 38th
  International Conference on Machine Learning, {ICML} 2021, 18-24 July 2021,
  Virtual Event}, volume 139 of \emph{Proceedings of Machine Learning
  Research}, 11920--11931. {PMLR}.

\bibitem[{Yarats, Kostrikov, and Fergus(2021)}]{DBLP:conf/iclr/YaratsKF21}
Yarats, D.; Kostrikov, I.; and Fergus, R. 2021.
\newblock Image Augmentation Is All You Need: Regularizing Deep Reinforcement
  Learning from Pixels.
\newblock In \emph{9th International Conference on Learning Representations,
  {ICLR} 2021, Virtual Event, Austria, May 3-7, 2021}. OpenReview.net.

\bibitem[{Yarats et~al.(2021{\natexlab{b}})Yarats, Zhang, Kostrikov, Amos,
  Pineau, and Fergus}]{DBLP:conf/aaai/Yarats0KAPF21}
Yarats, D.; Zhang, A.; Kostrikov, I.; Amos, B.; Pineau, J.; and Fergus, R.
  2021{\natexlab{b}}.
\newblock Improving Sample Efficiency in Model-Free Reinforcement Learning from
  Images.
\newblock In \emph{Thirty-Fifth {AAAI} Conference on Artificial Intelligence,
  {AAAI} 2021, Thirty-Third Conference on Innovative Applications of Artificial
  Intelligence, {IAAI} 2021, The Eleventh Symposium on Educational Advances in
  Artificial Intelligence, {EAAI} 2021, Virtual Event, February 2-9, 2021},
  10674--10681. {AAAI} Press.

\bibitem[{Zhang et~al.(2021)Zhang, McAllister, Calandra, Gal, and
  Levine}]{DBLP:conf/iclr/0001MCGL21}
Zhang, A.; McAllister, R.~T.; Calandra, R.; Gal, Y.; and Levine, S. 2021.
\newblock Learning Invariant Representations for Reinforcement Learning without
  Reconstruction.
\newblock In \emph{9th International Conference on Learning Representations,
  {ICLR} 2021, Virtual Event, Austria, May 3-7, 2021}. OpenReview.net.

\end{thebibliography}

\newpage
\appendix
\setcounter{theorem}{1}
\setcounter{proposition}{0}

\section{Theorems and Proofs}

\begin{theorem}
\label{theorem2_app}
Given a policy $\pi$, Simple State Representation (SimSR) which is updated as:
\begin{equation}
\begin{aligned}
    \label{ori_trans_app}
    \mathscr{F}^{\pi}\overline{\text{cos}}_{\phi}(\mathbf{x},\mathbf{y})=&|r_{\mathbf{x}}^{\pi}-r_{\mathbf{y}}^{\pi}|+\\\gamma &\mathbb{E} _{\mathbf{x}'\sim \hat{\mathcal{P}}_{\mathbf{x}}^{\pi},\mathbf{y}'\sim \hat{\mathcal{P}}_{\mathbf{y}}^{\pi}}[\overline{\text{cos}}_{\phi}(\mathbf{x}',\mathbf{y}')]
\end{aligned}
\end{equation}
has the same fixed point as MICo.
\end{theorem}

\begin{proof}
This proof mimics the proof of Proposition 4.2 from \cite{castro2021mico}. 
Let $\phi,\phi': \mathcal{X}\rightarrow\mathcal{S}$. Note that $\forall \mathbf{x},\mathbf{y}\in\mathcal{X}$: 
\begin{equation}
\begin{aligned}
&|\mathscr{F}^{\pi}\overline{\text{cos}}_{\phi}(\mathbf{x},\mathbf{y})-\mathscr{F}^{\pi}\overline{\text{cos}}_{\phi'}(\mathbf{x},\mathbf{y})|\\
=&\big|\gamma \mathbb{E} _{\mathbf{x}'\sim \hat{\mathcal{P}}_{\mathbf{x}}^{\pi},\mathbf{y}'\sim \hat{\mathcal{P}}_{\mathbf{y}}^{\pi}}[(\overline{\text{cos}}_{\phi}-\overline{\text{cos}}_{\phi'})(\mathbf{x},\mathbf{y})]\big| \\
\stackrel{(a)}{\leq}& \gamma \|(\overline{\text{cos}}_{\phi}-\overline{\text{cos}}_{\phi'})(\mathbf{x},\mathbf{y})\|_{\infty}
\end{aligned}
\end{equation}
We can easily verify that (a) holds since the transition is a stochastic matrix. Therefore, $\|\mathscr{F}^{\pi}\overline{\text{cos}}_{\phi}(\mathbf{x},\mathbf{y})-\mathscr{F}^{\pi}\overline{\text{cos}}_{\phi'}(\mathbf{x},\mathbf{y})\| \leq \gamma \|(\overline{\text{cos}}_{\phi}-\overline{\text{cos}}_{\phi'})(\mathbf{x},\mathbf{y})\|_{\infty}$.

Since ~\cite{castro2021mico} proved that Theorem~\ref{theorem:MICo_dis} has the unique fixed point and SimSR follows the same update procedure, we conclude that they share the same fixed point.
\end{proof}

\begin{theorem}
\label{simsr_with_dynamics_app}
Given a policy $\pi$, let SimSR operator $\mathbbm{F}^{\pi}:\mathbb{R}^{\mathcal{S}\times \mathcal{S}}\rightarrow\mathbb{R}^{\mathcal{S}\times \mathcal{S}}$ be
\begin{equation}
\begin{aligned}
    \label{con_trans_app}
    \mathbbm{F}^{\pi}\overline{\text{cos}}_{\phi}(\mathbf{x},\mathbf{y})=&|r_{\mathbf{x}}^{\pi}-r_{\mathbf{y}}^{\pi}|+\gamma \mathbb{E} _{\mathbf{s}'\sim \mathcal{P}_{\phi(\mathbf{x})}^{\pi},\mathbf{u}'\sim \mathcal{P}_{\phi(\mathbf{y})}^{\pi}}[\overline{\text{cos}}(s',u')].
\end{aligned}
\end{equation}
If latent dynamics are specified, $\mathbbm{F}^{\pi}$ has a fixed point.
\end{theorem}

\begin{proof}

If the latent dynamics are specified, then we can substitute $\mathbf{s}'$ with $\phi(\mathbf{x}')$ and substitute $\mathbf{u}'$ with $\phi(\mathbf{y}')$. Then we can easily mimic the proof of Theorem~\ref{theorem2_app} to complete this proof.
\end{proof}

\section{The benefits of SimSR operator}
In this section, we provide some theoretical discussion about the benefits of SimSR operator.
\subsection{Zero self-distance Effectiveness}

``Representation collapse'' in RL means that two states (observations) with different values are collapsed to the same representation. To investigate the connection between value function and state representation, the following proposition is applied:
\begin{proposition}\cite{castro2021mico}
\label{propo:representation_collapse_app}
$\forall \mathbf{x}, \mathbf{y} \in \mathcal{X}$ and for any policy $\pi\in\Pi$, 
\begin{equation}
    \begin{aligned}
&\left|V^{\pi}(\mathbf{x})-V^{\pi}\left(\mathbf{y}\right)\right| \leq U^{\pi}\left(\mathbf{x}, \mathbf{y}\right).
\end{aligned}
\end{equation}
\end{proposition}
Following Proposition~\ref{propo:representation_collapse_app}, we can obtain $\left|V^{\pi}(\mathbf{x})-V^{\pi}\left(\mathbf{y}\right)\right| \leq 1-\text{cos}(\phi(\mathbf{x}),\phi(\mathbf{y}))$ in SimSR. When two states are encoded into the same representation, the MICo distance $U^\pi(\mathbf{x},\mathbf{y})$ cannot guarantee their value being equal since it violates zero self-distance property. In contrast, when $\phi(\mathbf{x})=\phi(\mathbf{y})$ happens in SimSR, the cosine distance ensure that:
\begin{equation}
    \left|V^{\pi}(\mathbf{x})-V^{\pi}\left(\mathbf{y}\right)\right| \leq 1-\text{cos}(\phi(\mathbf{x}),\phi(\mathbf{y}))=0,
\end{equation}
where the transitivity of the relation guarantees their value function should be the same and therefore alleviates the aforementioned representation collapse issue.

\subsection{Computational complexity}
\begin{proposition}
The fixed point of $\pi$-bisimulation metric can be computed up to a prescribed degree of accuracy $\delta$ in $\tilde{O}\left(|\mathcal{S}|^{5} \log \delta / \log \gamma\right)$ operations with respect to $L^\infty$ norm. Given the fact that SimSR has the same fixed point with MICo, it can be computed in $O\left(|\mathcal{S}|^{4} \log \delta / \log \gamma\right)$ operations.
\end{proposition}
\begin{proof}
Refer to ~\cite{castro2021mico}.
\end{proof}

\begin{table*}
\caption{Comparison of key features in different algorithms. ``Distance consistency'' means that the distance computed in the representation space is consistent with the base ``metric'' that is used in ``behavioral difference''. ``-'' denotes ``does not apply''.}
\label{tab/compare_full}
\centering
\scriptsize
\begin{tabular}{ c c c c c c} 
 \toprule
    & Distance consistency & Unit length feature  & Zero self-distance &  Learn dynamics? \\
 \midrule
SimSR & \cmark & \cmark &  \cmark &  \cmark \\
DBC~\cite{DBLP:conf/iclr/0001MCGL21} & \xmark & \xmark & \cmark & \cmark \\
MICo~\cite{castro2021mico} & \cmark & \xmark & \xmark & \xmark \\
DeepMDP~\cite{DBLP:conf/icml/GeladaKBNB19} & - & \xmark& - & \cmark \\
 \bottomrule
\end{tabular}
\end{table*}

\section{Connection with Most Related Work}
The most related methods to our work are DBC~\cite{DBLP:conf/iclr/0001MCGL21}, MICo~\cite{castro2021mico}, and DeepMDP~\cite{DBLP:conf/icml/GeladaKBNB19}. Table~\ref{tab/compare_full}\footnote{Since DeepMDP does not actually learn representations by bisimulation metric (distance), it has some features as ``does not apply''. 
} provides the comparison of their key features.

DBC learns state representation by minimizing the mean square error between $\pi$-bisimulation metric and $\ell_1$ distance in the latent space:
\begin{equation}
\begin{aligned}
    J(\phi)=\Bigg(&\left\|\mathbf{z}_{i}-\mathbf{z}_{j}\right\|_{1}-\left|r_{i}-r_{j}\right|\\&-\gamma W_{2}\left(\hat{\mathcal{P}}\left(\cdot \mid \overline{\mathbf{z}}_{i}, \mathbf{a}_{i}\right), \hat{\mathcal{P}}\left(\cdot \mid \overline{\mathbf{z}}_{j}, \mathbf{a}_{j}\right)\right)\Bigg)^{2},
\end{aligned}
\end{equation}
where $\mathbf{z}_{i}=\phi(\mathbf{s}_i)$, $\mathbf{z}_{j}=\phi(\mathbf{s}_j)$, r denotes rewards, $\bar{\mathbf{z}}$ denotes $\phi(\mathbf{s})$ with stop gradients, and $\hat{\mathcal{P}}$ is the latent dynamics model that outputs Gaussian distribution. To update the representation, DBC first samples batch $B_i\sim \mathcal{D}$ from replay buffer, then uses permutation to get batch $B_j=\text{permute}(B_i)$, and finally trains the encoder: $\mathbb{E}_{B_i,B_j}[J(\phi)]$ which compares the instance-wise difference between states from different batches. 

There are three major differences between DBC and SimSR. First, DBC uses Euclidean distance of the Gaussian transition probability distribution in the latent space to compute the closed-form Wasserstein distance. As DBC tries to optimize $\ell_1$ distance between representations in the latent space, it has the risk of inconsistent and inaccurate approximation. In comparison, SimSR does not have this risk due to its design. 
Second, given the loss $\mathbb{E}_{B_i,B_j}[J(\phi)]$, for each update step in DBC, each single observation in a batch meet at most two other observations to update its representation. In comparison, SimSR improves the sample efficiency by extending the update to be matrix operation, where each observation can update its representation by computing the distance between every single observations in that batch and itself in one update step. 
Third, SimSR also normalizes the features and utilizes Huber loss to stabilize the representation, which is crucial in joint training. 

MICo is the most relevant work to ours, where a MICo distance is designed to measure the behavioral difference. The MICo distance is parameterized as follows:
\begin{equation}
\label{mico_param}
\begin{aligned}
    U^{\pi}(x, y) \approx U_{\omega}(x, y):=&\frac{\left\|\phi_{\omega}(x)\right\|_{2}+\left\|\phi_{\omega}(y)\right\|_{2}}{2}\\&+\beta \theta\left(\phi_{\omega}(x), \phi_{\omega}(y)\right)
\end{aligned}
\end{equation}
where $\phi_\omega$ is the mapping function parameterized by $\omega$, $\theta\left(\phi_{\omega}(x), \phi_{\omega}(y)\right)$ is the angle between vectors $\phi_\omega(x)$ and $\phi_\omega(y)$, and $\beta$ is a scalar hyperparameter. The settings of Eq.~\ref{mico_param} conform with the Proposition 4.10 in \cite{castro2021mico}. That is, MICo distance is a diffuse metric.
However, since the approximation involves the computation of the angle, numerical instability of the neural networks may escalate the changes of the angle and further affect the approximation precision. Thus, such approximation requires designing specific calculation method and adjusting hyper-parameter $\beta$ carefully according to different tasks, which may diminish the supposed practical advantage of MICo.
Besides, since $U_\omega$ violates the zero self-distance property, it may encounter the failure mode of representation collapse. In contrast, as SimSR is based on cosine distance and does not involve angle computation and additional hyper-parameters, it can naturally avoid representation collapse issue and is more stable when parameterized by neural networks. Furthermore, MICo only applies the raw environment transitions, while SimSR, instead, learns an ensemble version of latent dynamics to improve the robustness of representation, which is essential in environments that involve considerable uncertainty.

DeepMDP parameterizes the reward model and transition model to learn state representation. Its algorithms rely on a rather strong assumption that the learned representation is Lipschitz. In comparison with DeepMDP, SimSR is a simpler framework with significantly better performance as shown in the experiments.

\section{Additional Results}
\subsection{Default setting}
Figure~\ref{fig/standard_curve_app} illustrates the results of experiments with the default setting. 

\begin{figure*}[h]
\centering
\subfigure[Ball In Cup Catch]{
\includegraphics[width=0.25\textwidth]{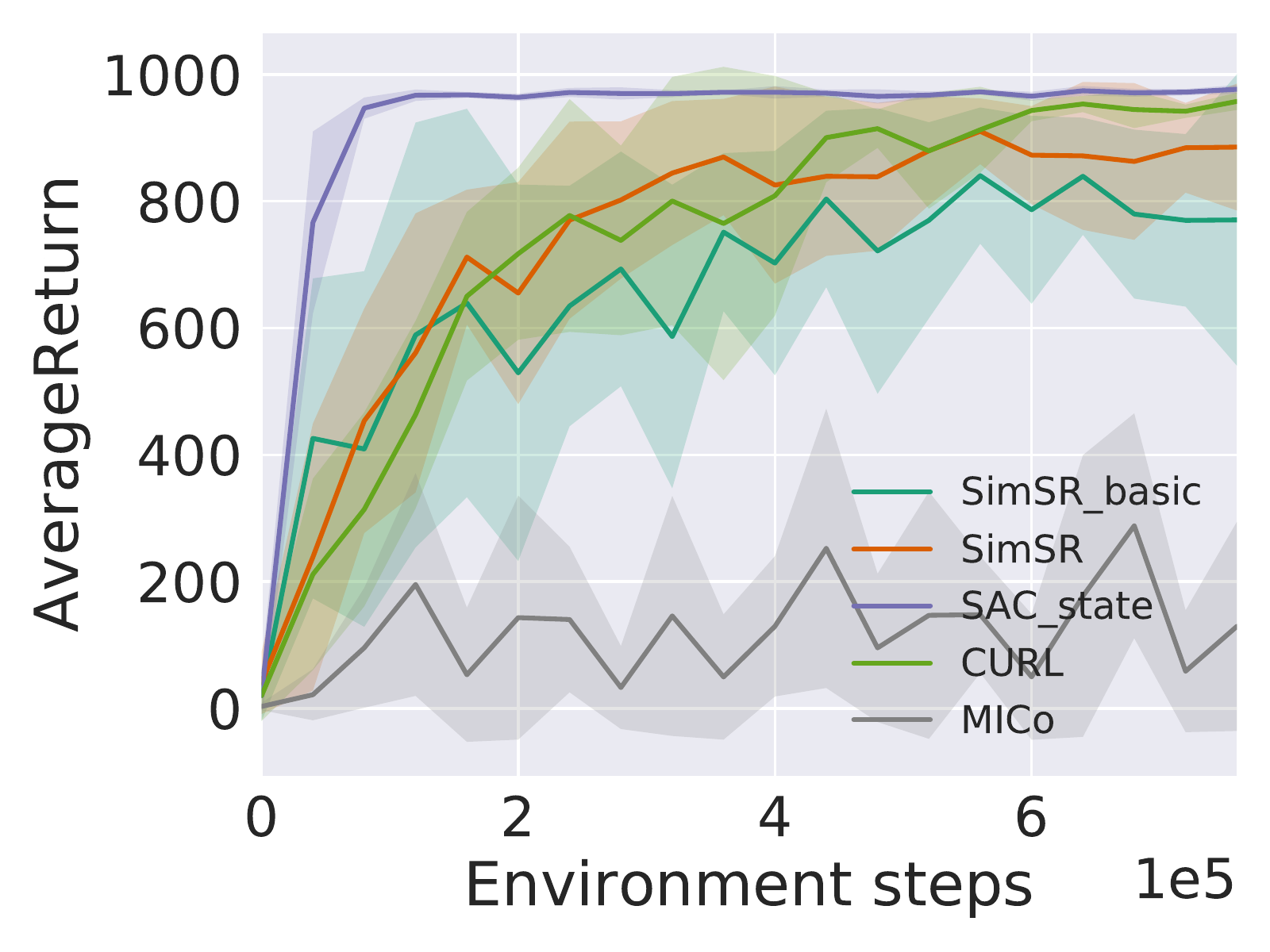}
}
\subfigure[Cartpole Swingup]{
\includegraphics[width=0.25\textwidth]{figures/exp_basic/cartpole_swingup.pdf}
}
\subfigure[Cartpole Swingup Sparse]{
\includegraphics[width=0.25\textwidth]{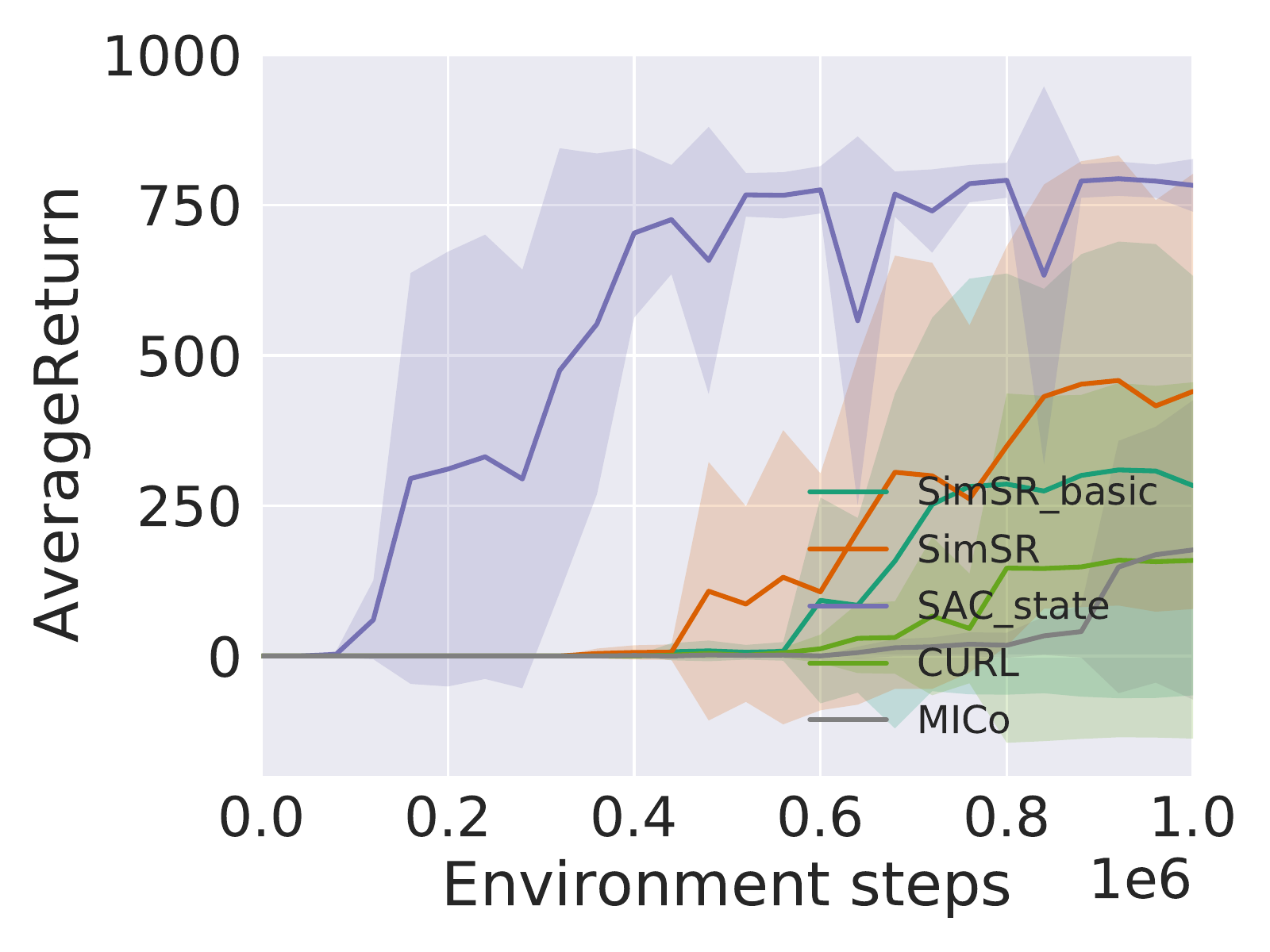}
}
\subfigure[Cheetah Run]{
\includegraphics[width=0.25\textwidth]{figures/exp_basic/cheetah_run.pdf}
}
\subfigure[Hopper Hop]{
\includegraphics[width=0.25\textwidth]{figures/exp_basic/hopper_hop.pdf}
}
\subfigure[Hopper Stand]{
\includegraphics[width=0.25\textwidth]{figures/exp_basic/hopper_stand.pdf}
}
\subfigure[Finger Spin]{
\includegraphics[width=0.25\textwidth]{figures/exp_basic/finger_spin.pdf}
}
\subfigure[Pendulum Swingup]{
\includegraphics[width=0.25\textwidth]{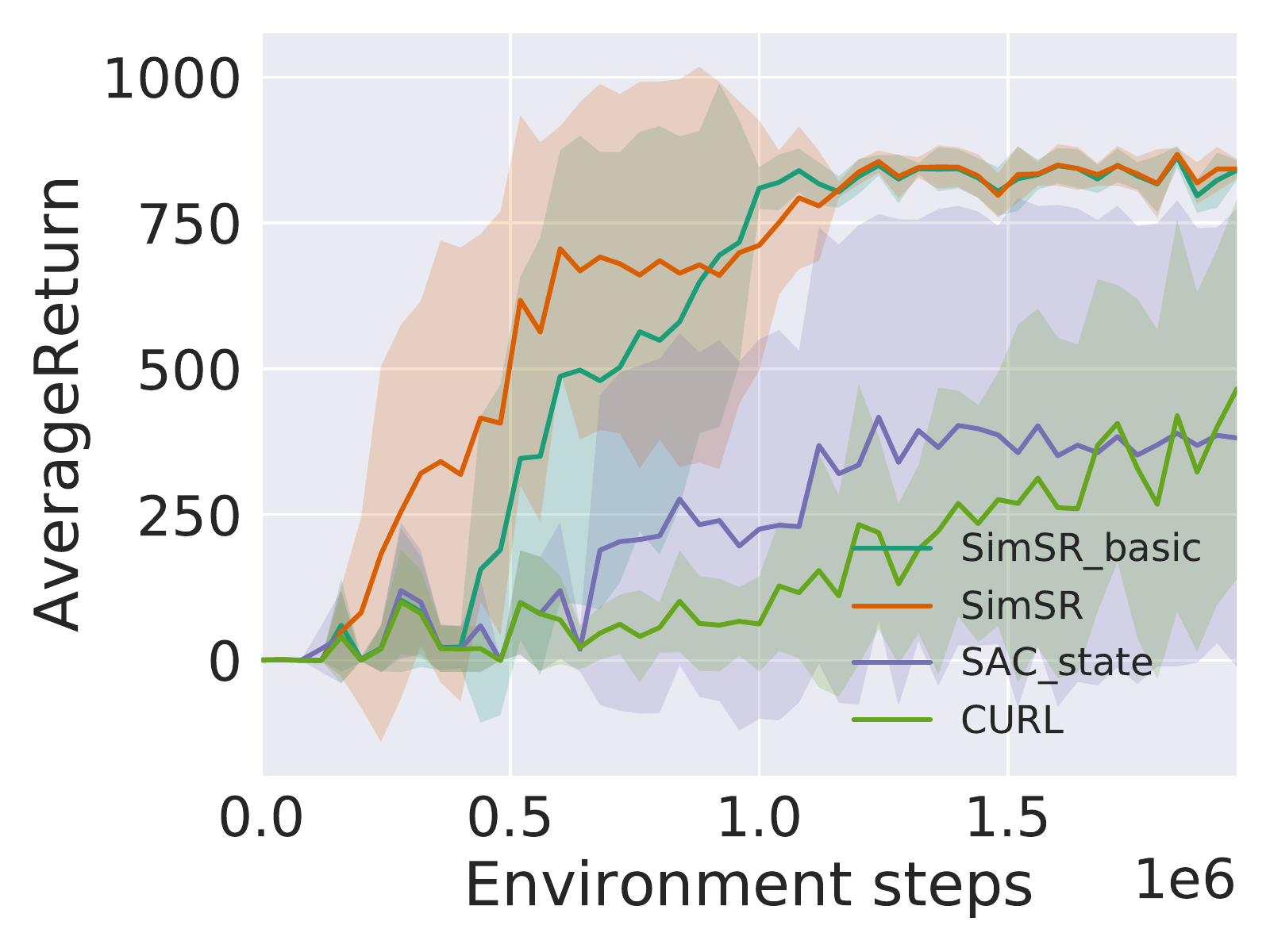}
}
\subfigure[Walker Walk]{
\includegraphics[width=0.25\textwidth]{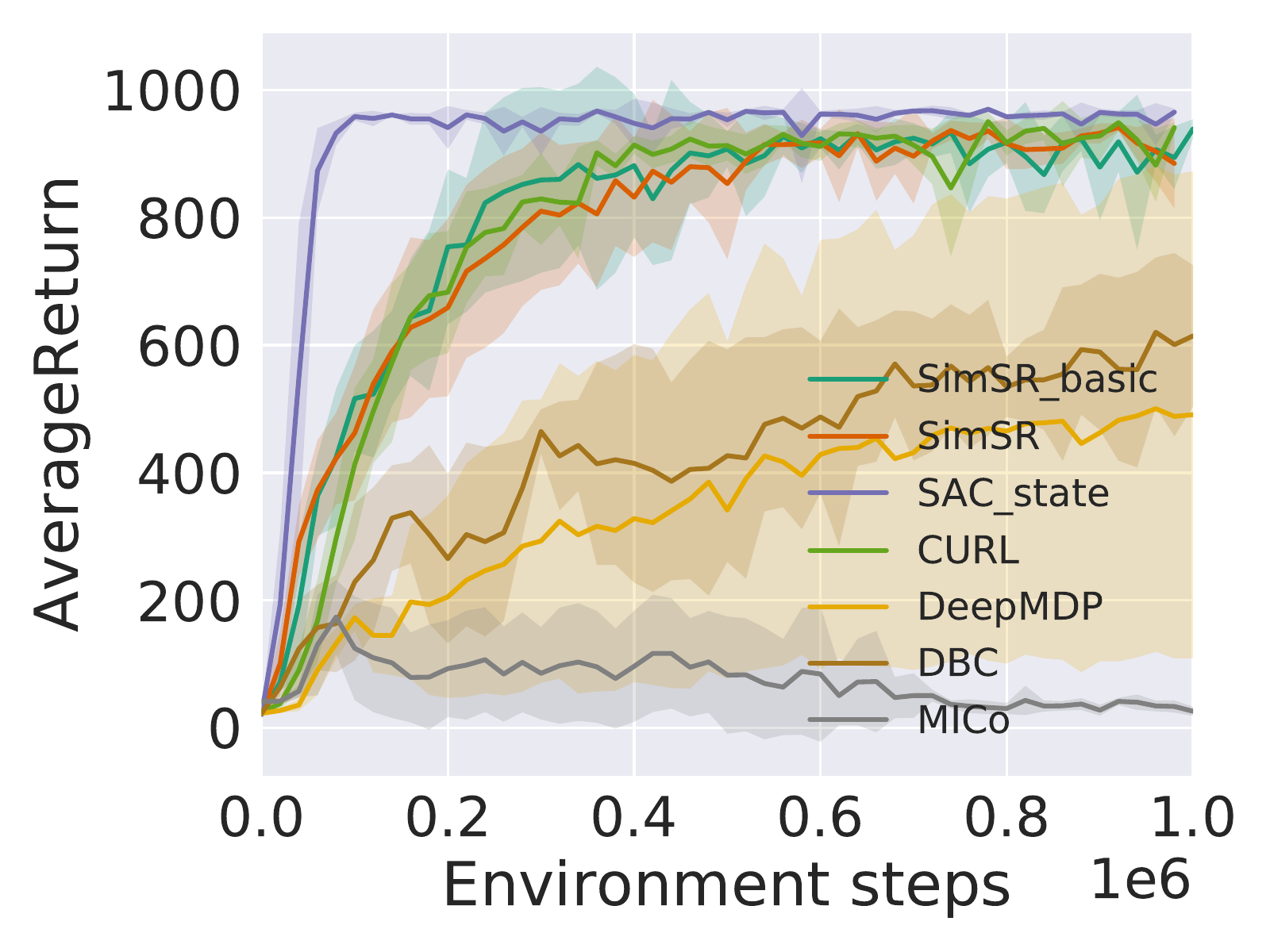}
}
\caption{Performance comparison on 9 DMC tasks over 5 seeds with one standard error shaded in the default setting. For every seed, the average return is computed every 10,000 training steps, averaging over 10 episodes. The horizontal axis indicates number of environment steps. The vertical axis indicates the average return.  }
\label{fig/standard_curve_app}
\end{figure*}

\subsection{Natural video setting}
t-SNE was applied to visualize the representations learned by SimSR and CURL. Figure~\ref{fig:TSNE} shows that regardless of extremely different backgrounds, SimSR is capable of filtering out irrelevant information and mapping the observations with similar robot configurations to neighbors closing to each other in the representation space. 
The color represents the value of the reward for each representation. SimSR learns representations that are close in the latent space with similar reward values. In contrast, CURL maps similar physical-state image pairs to points far from each other in the embedding space.

\begin{figure*}[htbp]
\centering
\includegraphics[width=1\textwidth]{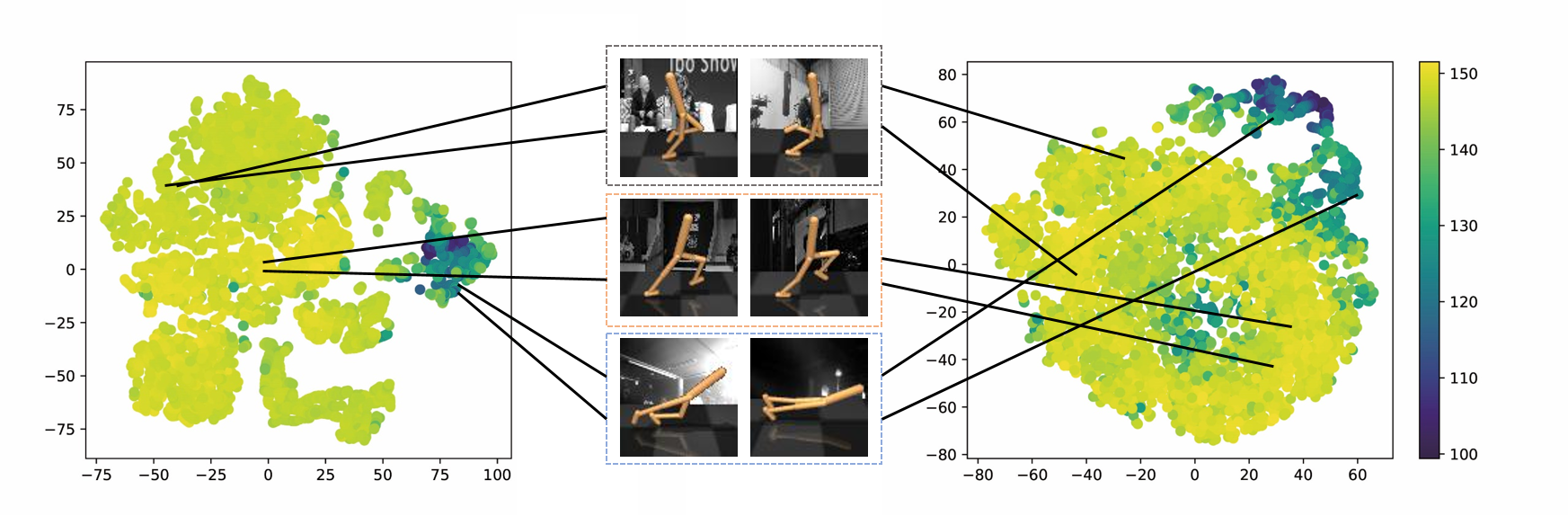}
\caption{t-SNE plots of latent spaces learned with SimSR (left) and CURL (right). Color represents the reward value of embedded points (yellow for higher value yellow and purple for lower value). Each pair of lines indicates the corresponding embedded points for observations with similar physical states but different backgrounds.}
\label{fig:TSNE}
\end{figure*}

\subsection{Ablation study results}
Figure~\ref{fig/ablation} illustrates the results of ablation experiments. 

\begin{figure}[t]
\centering
\subfigure[Ball In Cup Catch]{
\includegraphics[width=0.22\textwidth]{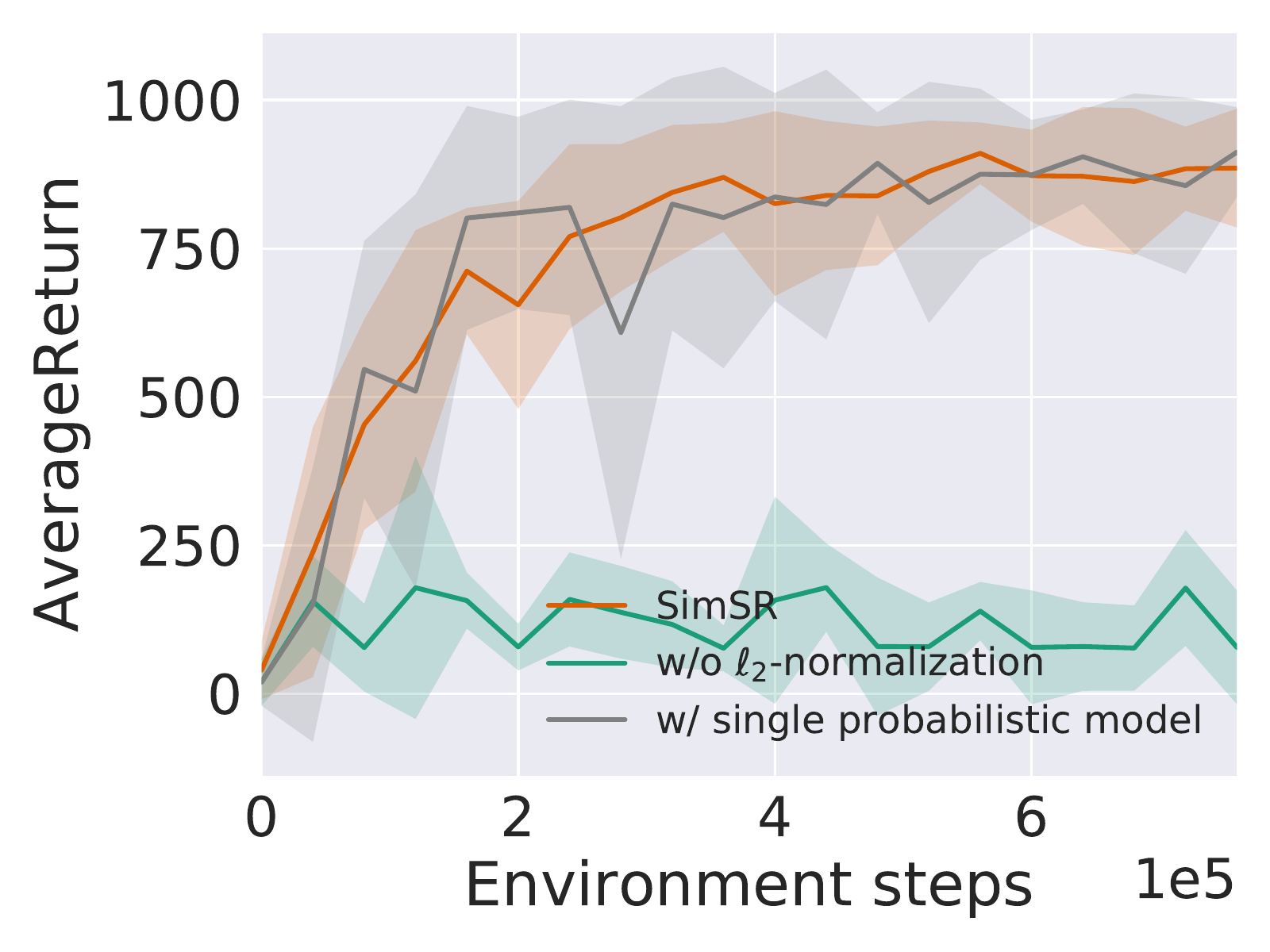}
}
\subfigure[Cheetah Run]{
\includegraphics[width=0.22\textwidth]{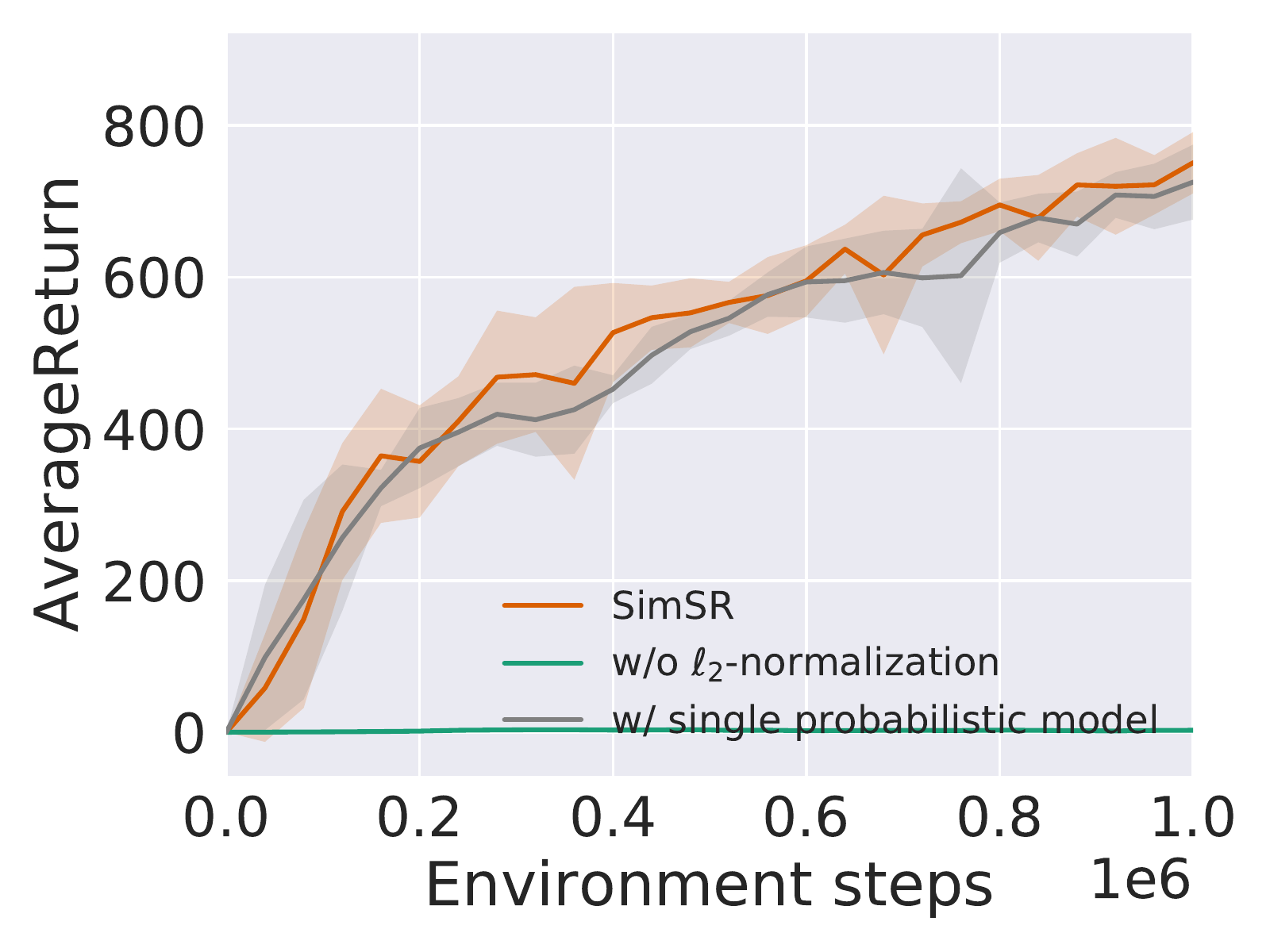}
}
\subfigure[Hopper Hop]{
\includegraphics[width=0.22\textwidth]{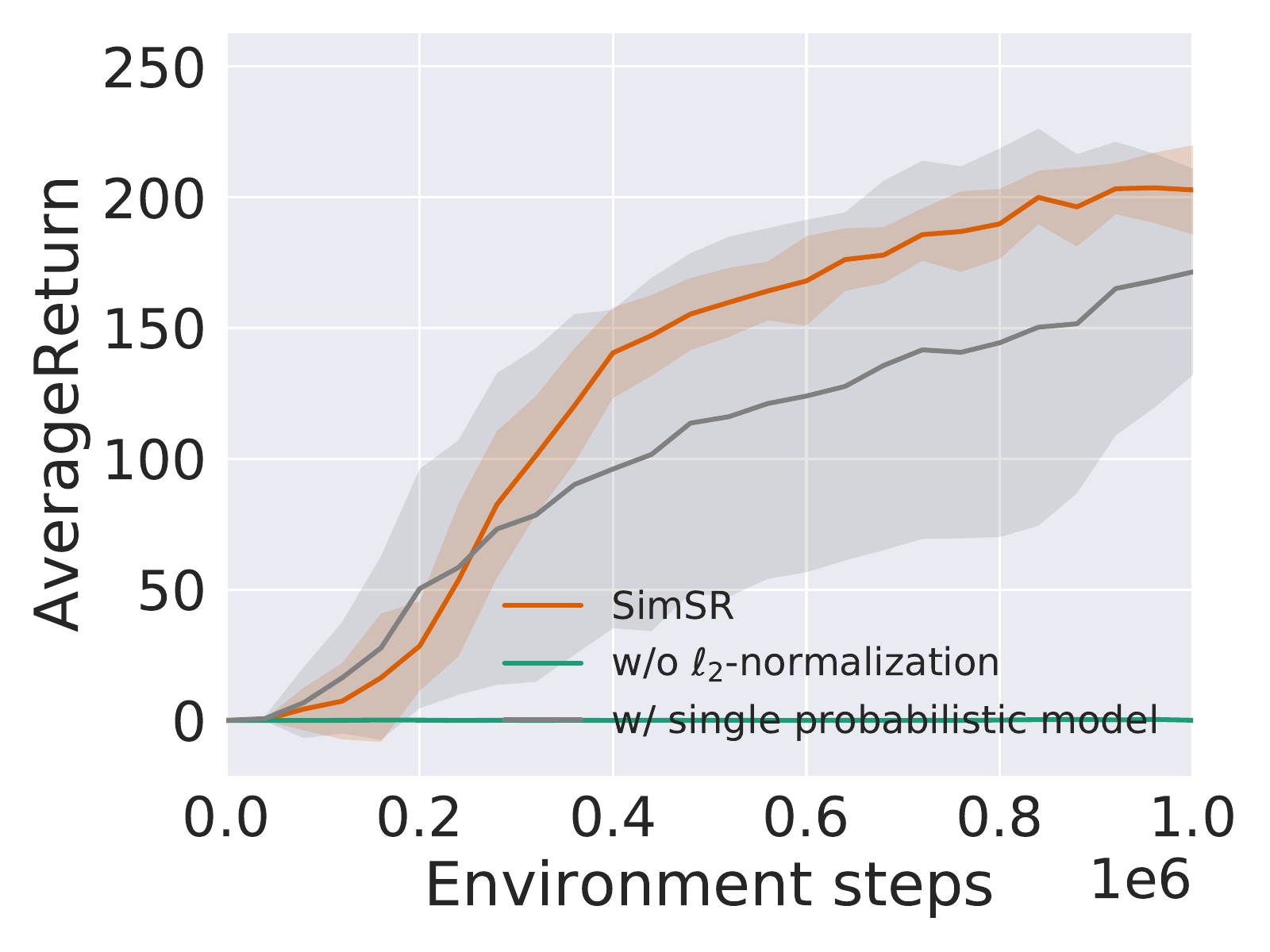}
}
\subfigure[Finger Spin]{
\includegraphics[width=0.22\textwidth]{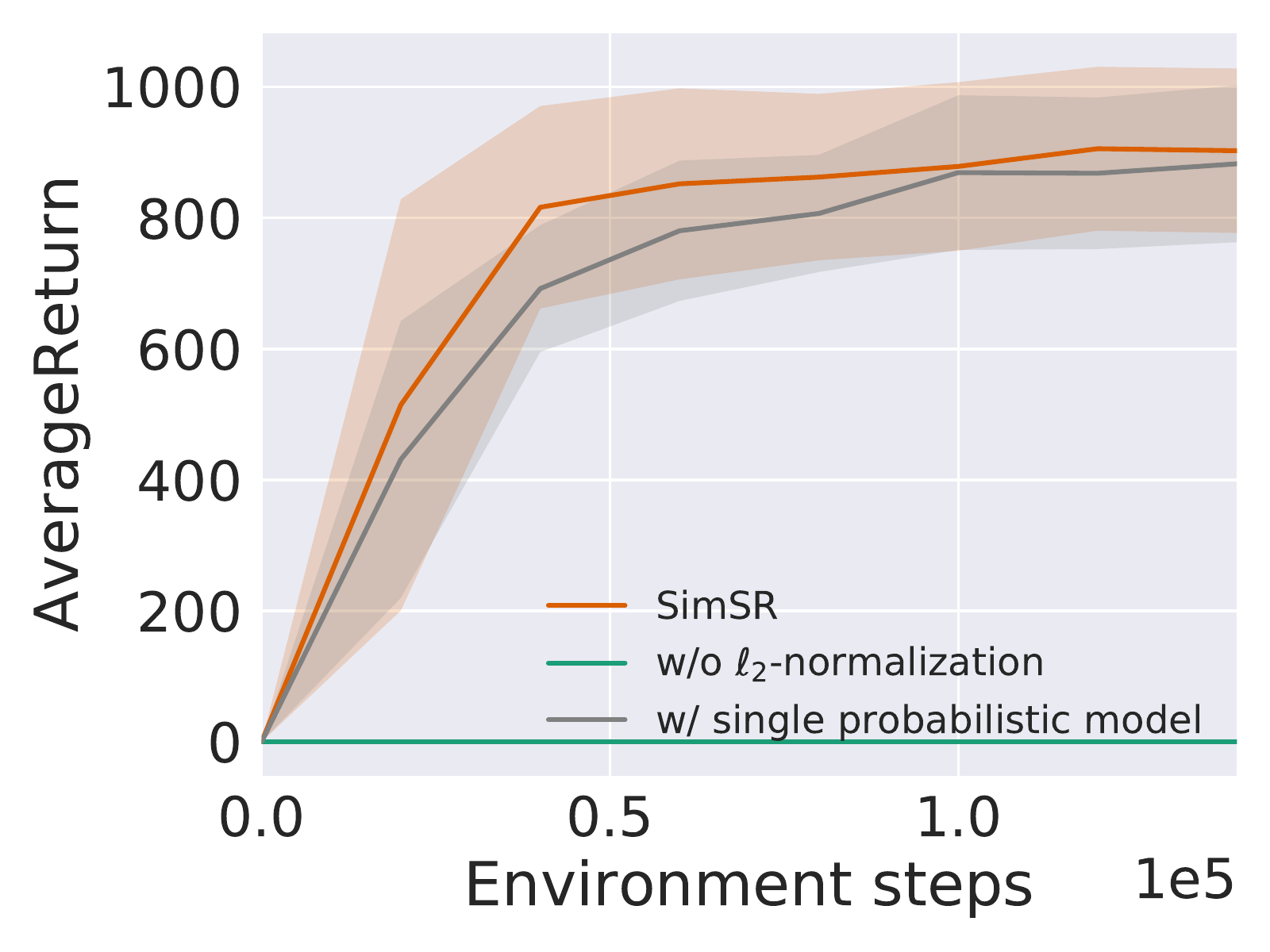}
}
\caption{Results of ablation study}
\label{fig/ablation}
\end{figure}

 \section{Implementation Details}
We build the agents by combining our model with soft actor critic (SAC) algorithm~\cite{DBLP:journals/corr/abs-1812-05905} to devise a practical reinforcement learning method. Table~\ref{table:hyperparameters} list the hyperparameters used in the experiments. The pseudocode of the encoder and transition model architecture is listed at the end of the appendix.

\begin{table}[h]
\caption{Hyperparameters used for DMC experiments.}
\label{table:hyperparameters}
\vskip 0.15in
\begin{center}
\begin{small}
\begin{tabular}{ll}
\toprule
\textbf{Hyperparameter} & \textbf{Value}  \\
\midrule
Observation shape    & $(84,84)$  \\ 
Replay buffer size    & $100000$ \\ 
Initial steps    & $1000$  \\ 
Stacked frames    & $3$  \\ 
Action repeat    & $2$ finger, spin; walker, walk\\
 & $8$ cartpole, swingup \\
 & $4$ otherwise  \\
Hidden units (MLP)    & $1024$  \\ 
Evaluation episodes    & $10$  \\ 
Optimizer    & Adam  \\ 
$(\beta_1,\beta_2) \rightarrow (f_\theta, \pi_\psi, Q_\phi)$   & $(.9,.999)$  \\
$(\beta_1,\beta_2) \rightarrow (\alpha)$   & $(.5,.999)$  \\
Learning rate & $1e-4$ $\alpha$ \\
  & $1e-3$ otherwise\\ 

Batch size    & $128$  \\ 
$Q$ function EMA $\tau$ & $0.01$ \\
Critic target update freq & $2$ \\
Convolutional layers & $4$ \\
Number of filters & $32$ \\
Non-linearity & ReLU \\
Encoder EMA $\tau$ & $0.05$ \\
Latent dimension & $50$ \\
Discount $\gamma$ & $.99$ \\
Initial temperature & $0.1$ \\

\bottomrule
\end{tabular}
\end{small}
\end{center}
\vskip -0.1in
\end{table}

\subsubsection{Network architecture}
The pseudocode of the encoder and transition model architecture is listed below. The actor and critic both use the same encoder to embed image observations.
 
Pseudocode of the encoder:
\begin{lstlisting}[language=Python]

def encode(x,z_dim, num_layers):
    """
    ConvNet encoder
    
    B-batch_size, C-channels,  H,W-spatial_dims
    C = c * num_frames; c = 3 (R/G/B) or 1 (gray)
    
    args:
        x: shape : [B, C, H, W]
        z_dim: latent dimension
        num_layers: numbers of convolutional layers
    """

    # c: channels, f: filters
    # k: kernel, s: stride
    z = Conv2d(c=x.shape[1], f=32, k=3, s=2)])(x)
    z = ReLU(z)
    
    for _ in range(num_layers - 1):
        z = Conv2d((c=32, f=32, k=3, s=1))(z)
        z = ReLU(z)
    
    z = flatten(z)
    
    # in: input dim, out: output_dim, h: hiddens
    z = mlp(in=z.size(),out=z_dim,h=1024)
    # l2 normalize
    z = torch.nn.functional.normalize(h_fc, dim=1, p=2)
\end{lstlisting}

Pseudocode of the transition model:
\begin{lstlisting}[language=Python]
def transition(x,z_dim, log_std_min, log_std_max):
    """
    transition model

    B-batch_size
    M-state dim (z_dim)
    N-action dim

    args:
        x : shape: [B, M+N]
        z_dim: latent dimension
        log_std_min: lower bound of log std
        log_std_max: upper bound of log std
    """

    # in: input dim, out: output_dim, h: hiddens
    mu = mlp(in=x.size(),out=z_dim,h=1024)
    log_std = mlp(in=x.size(),out=z_dim,h=1024)

    log_std = tanh(log_std)
    # Enforce log std bounded
    log_std = log_std_min + 0.5 * (
            log_std_max - log_std_min
    ) * (log_std + 1)

    sigma = log_std.exp()
    
\end{lstlisting}

Pseudocode of the SimSR loss:
\begin{lstlisting}[language=Python]

def update_encoder(self, obs, next_obs, rewards, actions, gamma=0.99):
    """
    SimSR loss

    f_q, f_k: encoder networks for current 
    observations and next observations respectively.
    f_t: transition probability network.
    N: latent dimension of states
    D: dimension of actions

    obs: current observations, shape: [B, N]
    next_obs: next observations, shape: [B, N]
    rewards:  shape: [B, 1]
    actions:  shape: [B, D]
    gamma: discounted factor
    """
        z_obs = f_q.encode(obs)
        with torch.no_grad():
            z_obs_tmp = f_k.encode(obs)
            z_next_obs = f_t.sample_prediction(
                torch.cat([z_obs_tmp, actions], dim=1))
        # shape: [B, B]
        SimSR_distance = 1 - torch.matmul(z_obs, z_obs.T)

        with torch.no_grad():
            r_diff = torch.abs(rewards.T - rewards)
            next_obs_diff = 1 - torch.matmul(z_next_obs, z_next_obs.T)
            target_SimSR_distance = r_diff + gamma * next_obs_diff

        SimSR_loss = nn.loss.huberloss(SimSR_distance, target_SimSR_distance)
        return loss
\end{lstlisting}

\end{document}